\documentclass[final]{siamltex}

\usepackage[final]{fixme}

\usepackage{amssymb}
\usepackage{relsize}
\usepackage{algpseudocode}
\usepackage{algorithm}
\usepackage{subfigure}
\usepackage{url}
\usepackage{color}
\usepackage[utf8]{inputenc}
\usepackage{enumitem}
\usepackage{amsmath}
\usepackage{amsfonts}
\usepackage{amssymb}
\usepackage{mathrsfs}
\usepackage{graphicx}
\usepackage{stackrel}

\input xy
\xyoption{all}

\newcommand{\RR}{\mathbb{R}}

\newcommand{\Id}{\mathrm{Id}}

\renewcommand{\phi}{\varphi}

\newcommand{\ip}[1]{\left<#1\right>}
\newcommand{\apply}[2]{\big(#1\big|#2\big)}

\title{Higher-Order Momentum Distributions and Locally Affine LDDMM Registration}

\author{Stefan Sommer\thanks{Dept. of Computer Science, Univ. of Copenhagen, Denmark ({\tt sommer@diku.dk})} 
\and Mads Nielsen$\,^{*,}$\thanks{BiomedIQ, Copenhagen, Denmark} 
\and Sune Darkner$\,^*$ 
\and Xavier Pennec\thanks{Asclepios Project-Team, INRIA Sophia-Antipolis, France}}

\begin{document}

\maketitle

\begin{abstract} % 250 words
  To achieve sparse parametrizations that allows intuitive analysis, we aim to
  represent deformation with a basis containing \emph{interpretable} elements, 
  and we wish to use elements that have the description capacity
  to represent the deformation \emph{compactly}.
  To accomplish this, we introduce in this paper \emph{higher-order momentum distributions} in
  the LDDMM registration framework. 
  While the zeroth order moments previously used in LDDMM only describe local displacement, 
  the first-order momenta that are proposed here represent a basis that 
  allows local description of
  affine transformations and subsequent compact description of non-translational movement
  in a globally non-rigid deformation. The resulting 
  representation contains directly interpretable information from both 
  mathematical and modeling perspectives.
  We develop the mathematical construction of the registration framework with higher-order momenta, we
  show the implications for \emph{sparse} image registration and deformation description, and we
  provide examples of how the parametrization enables registration
  with a very low number of parameters.
  The capacity and interpretability of the parametrization using higher-order momenta lead to
  natural modeling of articulated movement, and the method promises to be useful
  for quantifying ventricle expansion and progressing atrophy during Alzheimer's disease.

\end{abstract}


\pagestyle{myheadings}
\thispagestyle{plain}

\section{Introduction}
In many image registration applications, we
wish to describe the deformation 
using as few parameters as possible and with a representation that 
allows intuitive analysis: we search for
parametrizations with basis elements that have the \emph{capacity} to
describe deformation \emph{sparsely} while being directly \emph{interpretable}.
For instance, we wish to use such a representation to compactly describe
the progressive atrophy that occurs in the human brain suffering from
Alzheimer's disease and that can be detected by the expansion of the ventricles
\cite{jack_medial_1997,fox_presymptomatic_1996}.

Image registration algorithms often represent
translational movement in a dense sampling of the image domain. Such approaches
fail to satisfy the above goals: low dimensional
deformations such as expansion of the ventricles will not be represented
sparsely; the registration algorithm must optimize a large
number of parameters; and the expansion cannot easily be interpreted from the
registration result.

In this paper, we use \emph{higher-order momentum distributions} in the LDDMM registration 
framework to obtain a deformation parametrization that
increases the \emph{capacity} of sparse approaches with a basis consisting
of \emph{interpretable} elements.
We show how the higher-order representation model locally affine
transformations, and
we use the compact deformation description 
to register points and images using very few parameters. We illustrate how
the deformation coded by the higher-order momenta can be directly interpreted
and that it represents information directly useful in applications: with low numbers of
control points, we can detect the expanding ventricles of the patient shown in 
Figure~\ref{fig:atrophy1}.
\begin{figure}[t]
\begin{center}
  \parbox{0.99\columnwidth}{
  \subfigure[Baseline with control points.]{\includegraphics[width=0.48\columnwidth,trim=0 0 0 0,clip=true]{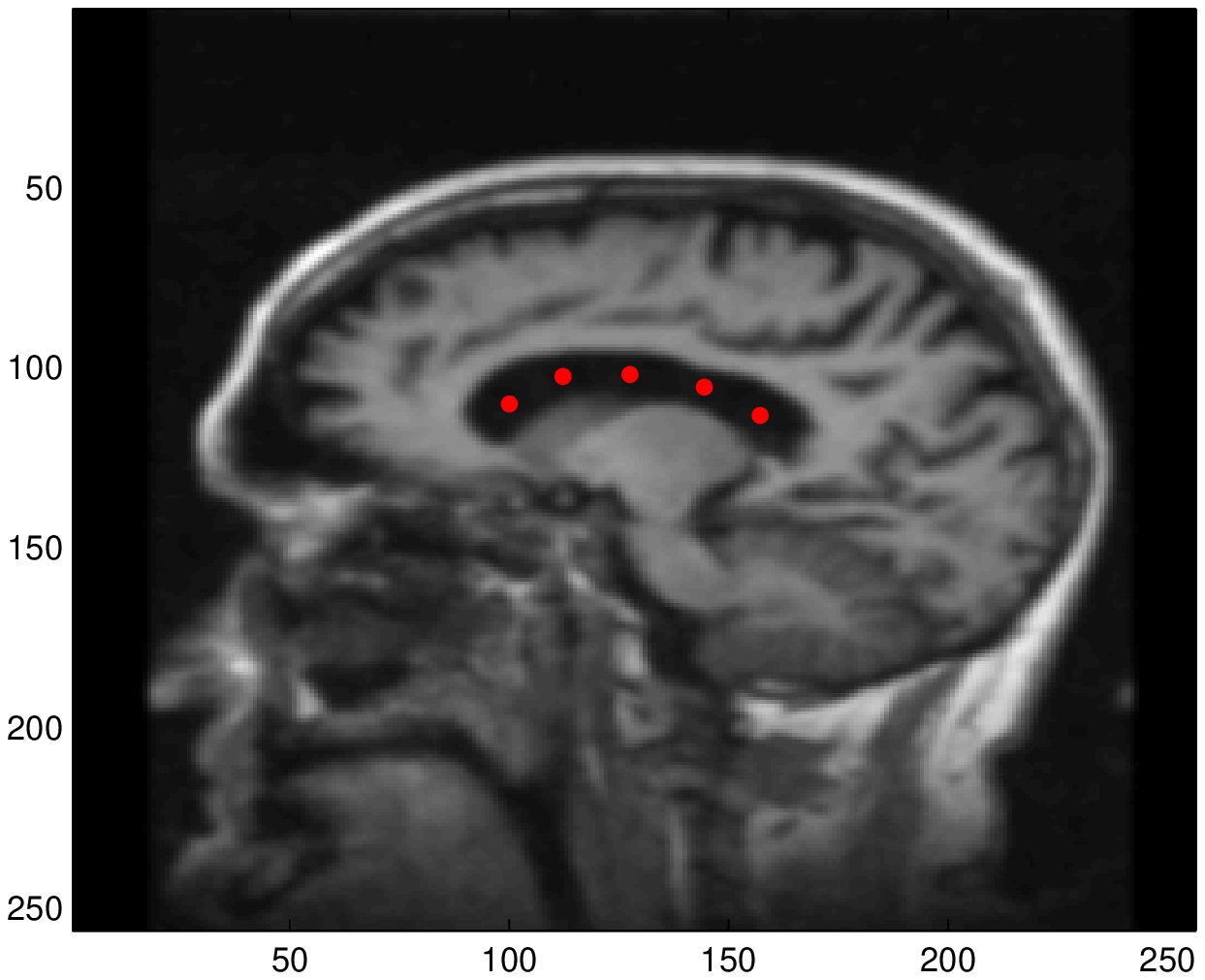}}
  \subfigure[Follow up (box marking zoom area, figure (c) and (d)).]{\includegraphics[width=0.48\columnwidth,trim=0 0 0 0,clip=true]{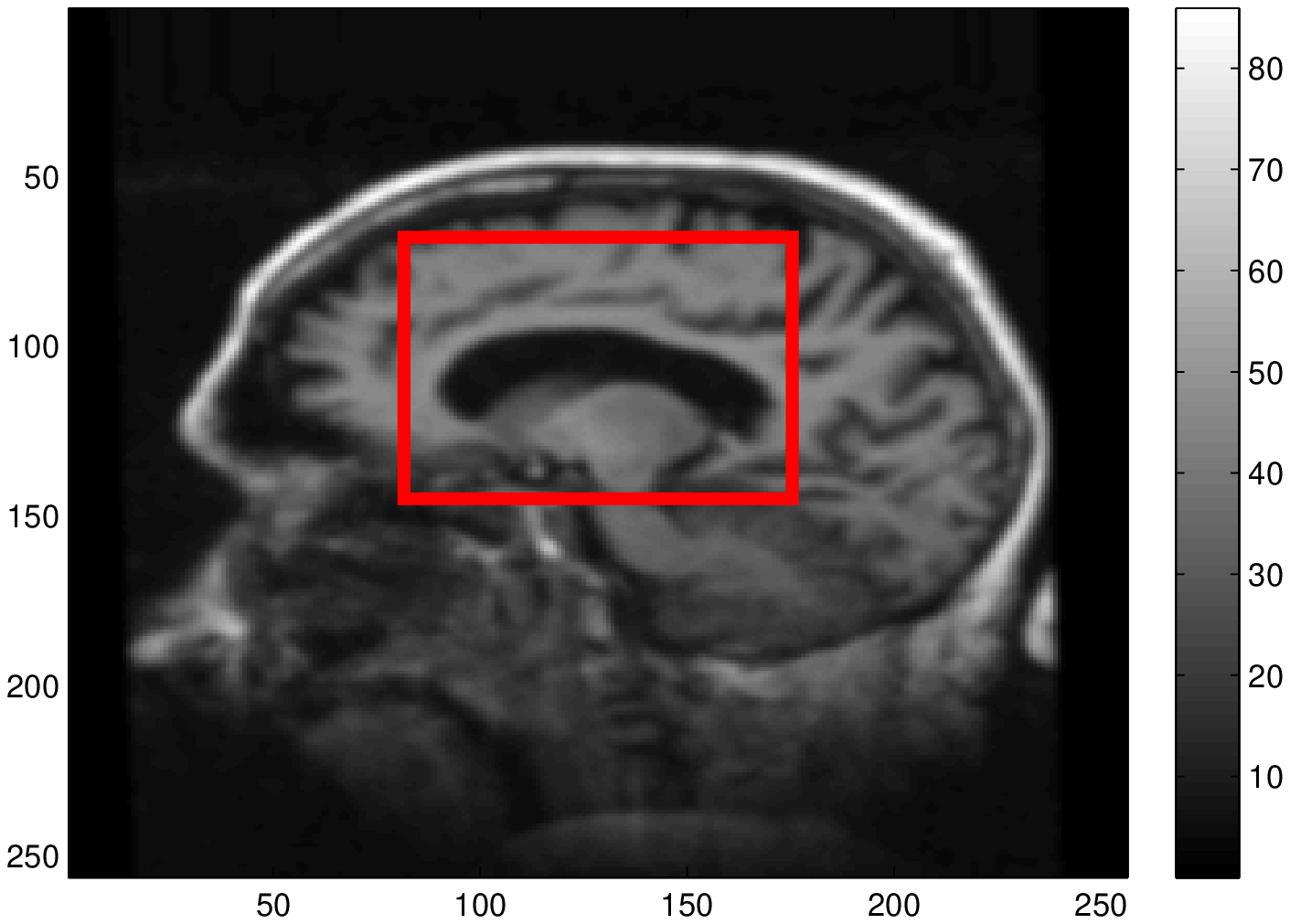}}
  %\subfigure[Difference, ventricle area]{\includegraphics[width=0.49\columnwidth,trim=0 0 0 0,clip=true]{figures/diff-baseline-followup}}
  \subfigure[$\log$-Jacobian in ventricle area.]{\includegraphics[width=0.47\columnwidth,trim=0 0 0 0,clip=true]{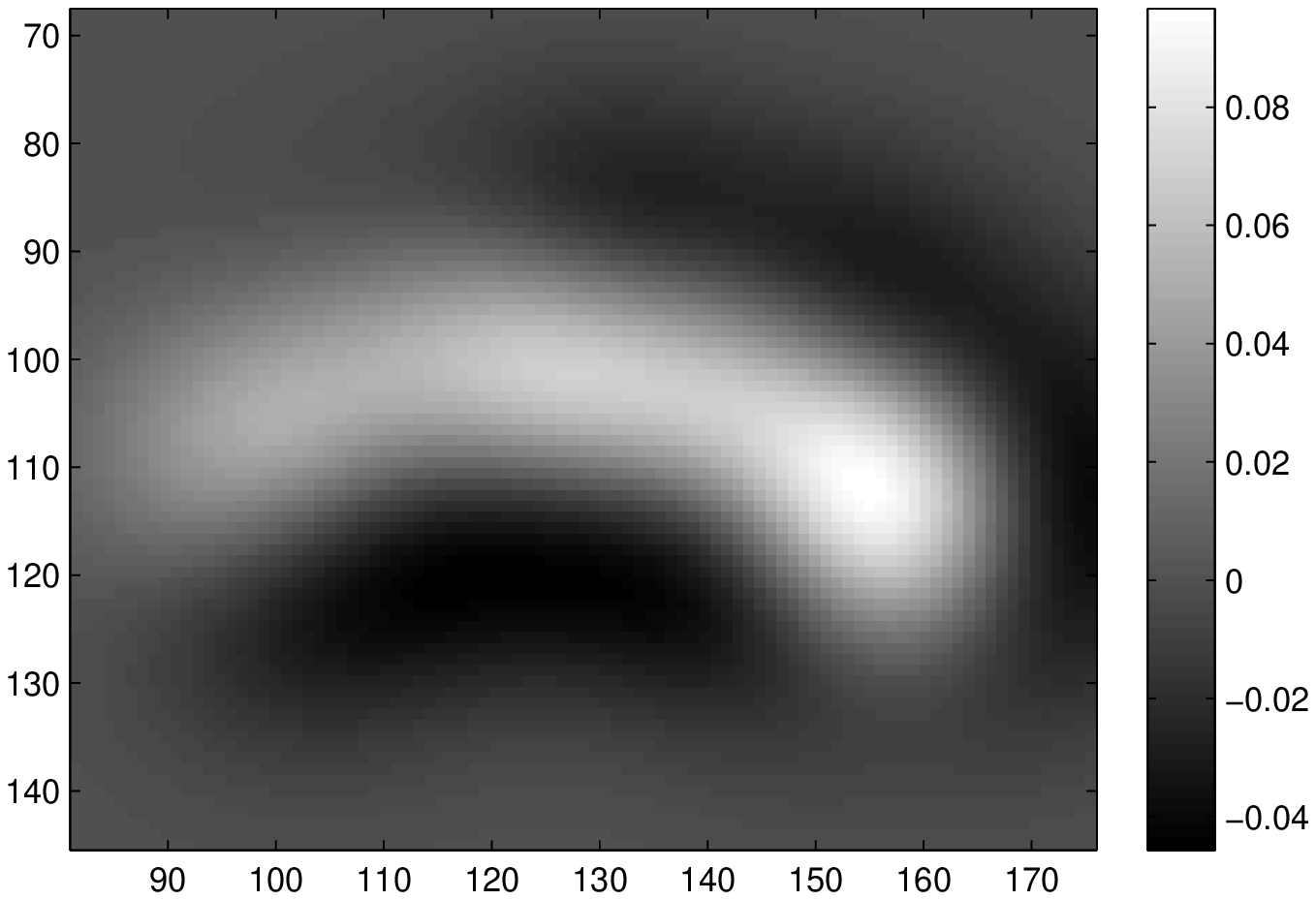}} 
  \subfigure[Initial deformation field in ventricle area.]{\includegraphics[width=0.49\columnwidth,trim=-5 -10 0 0,clip=true]{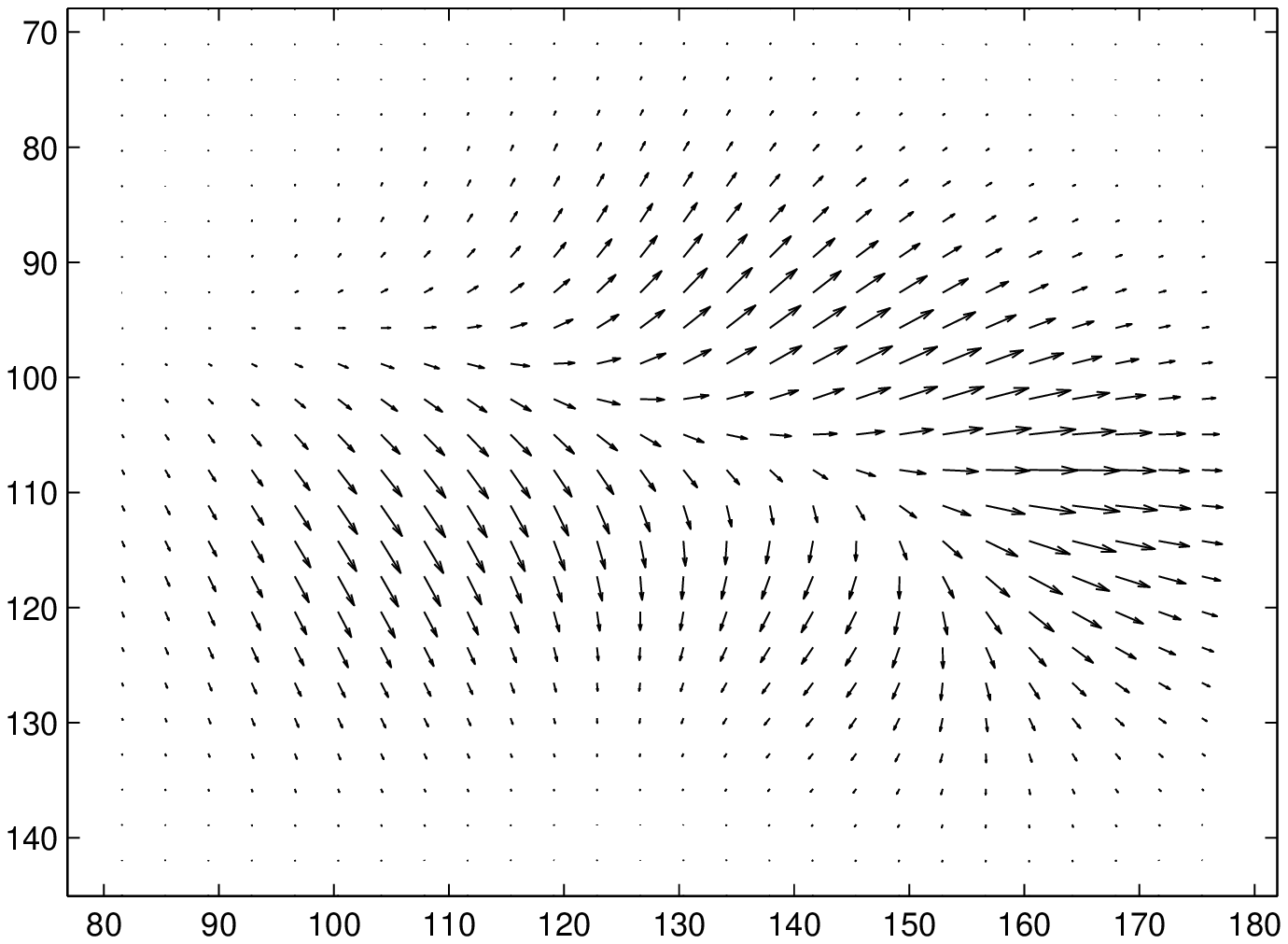}}
  }
\end{center}
\caption{Progressing Alzheimer's disease cause atrophy and expansion of the
ventricles. By placing five \emph{deformation atoms} in the ventricle area of
the baseline MRI scan \cite{marcus_open_2010} and by using \emph{higher-order
momenta}, we can detect the expansion. (a) The position of the
deformation atoms shown in the baseline scan; (b) the follow up scan; 
(c) the $\log$-Jacobian determinant of the generated deformation in the ventricle area
(red box in (b)); (d) the vector field at $t=0$ of the generated deformation.
The logarithm of the Jacobian determinant and the divergence at the deformation
atoms are positive which is in line with the expected 
ventricle expansion, confer also Figure~\ref{fig:atrophy2}.
}
\label{fig:atrophy1}
\end{figure}

\subsection{Background}
Most of the methods for non-rigid registration in medical imaging
model the displacement of each spatial position by either a
combination of suitable basis functions for the displacement itself or for the
velocity of the voxels. The number of control points vary between one for each
voxel \cite{arsigny_log-euclidean_2006,hernandez_registration_2009,christensen_deformable_2002}
and fewer with larger basis functions
\cite{rueckert_nonrigid_1999,bookstein_linear_1999,durrleman_optimal_2011}. For
all methods, the infinite-dimensional space of deformations is approximated by the
finite- but high-dimensional subspace spanned by the parametrization of the
individual method. The approximation will be good if the underlying deformation 
is close to this subspace, and the representation will be compact, if few basis
functions describe the deformation well. The choice of basis functions play a 
crucial role, and we will in the rest of the paper denote them \emph{deformation
atoms}. Two main observations constitute the motivation for the work presented 
in this paper:

\emph{Observation 1: Order of the Deformation Model.}
In the majority of registration methods, the deformation atoms model the local \emph{translation} of each
point. We wish a richer representation which is in particular able to model
locally linear components in addition to local translations.
The Polyaffine and Log-Euclidean Polyaffine
\cite{arsigny_polyrigid_2005,arsigny_fast_2009} frameworks pursue this by
representing the velocity of a path of deformations locally by matrix logarithms.
Ideas from the Polyaffine methods have recently been incorporated in e.g.
the Demons algorithm \cite{vercauteren_diffeomorphic_2009} but, to the best of
our knowledge, not in the LDDMM registration framework.
We wish to extend the set of deformation atoms used in LDDMM to allow
representation of \emph{first-} and \emph{higher-order} structure and
hence incorporate the benefits of the Polyaffine methods in the LDDMM framework.

\emph{Observation 2: Order of the Similarity Measure.}
When registering DT images, the reorientation is a function of the derivative
of the warp; curve normals also contain directional information which is dependent on the 
warp derivative and airway trees contain directional information in the tree
structure which can be used for measuring similarity. These are examples
of similarity measures containing \emph{higher-order information}.
For the case of image registration, the warp derivative may also enter the equation 
either directly in the similarity measure
\cite{roche_rigid_2001,pluim_image_2000} or
to allow use of more image information than provided by a sampling of
the warp. Consider an image similarity measure on the form
$U(\phi)=\int_\Omega F(I_m(\phi^{-1}(x)),I_f(x))dx$.
A finite sampling of the domain $\Omega$ can approximate this with
\begin{equation*}
  \tilde{U}^0(\phi)
  =
  \frac{1}{N}\sum_{k=1}^NF(I_m(\phi^{-1}(x_k)),I_f(x_k))
  \ .
\end{equation*}
Letting $\{p_1,\ldots,p_P\}$ be uniformly distributed points around $0$, we can
increase the amount of image information used in $\tilde{U}^0(\phi)$
\emph{without} additional sampling of the warp by
using a first-order approximation of $\phi^{-1}$:
\begin{equation*}
  \tilde{U}^1(\phi)
  =
  \frac{1}{NP}\sum_{k=1}^N\sum_{l=1}^NF(I_m(D\phi^{-1}p_l+\phi^{-1}(x_k)),I_f(p_l+x_k))
  \ .
\end{equation*}
This can be considered an increase from \emph{zeroth} to \emph{first-order}
in the approximation of $U$. Besides including more image information
than provided by the initial sampling of the warp, the increase in order allows capture of
non-translational information - e.g. rotation and dilation - in the similarity
measure. The approach can be seen as a specific case of similarity smoothing; more examples 
of smoothing in intensity based image registration
can be found in \cite{darkner_generalized_2011}.
\ \\

We focus on deformation modeling with the Large Deformation Diffeomorphic Metric Mapping
(LDDMM) registration framework which has the
benefit of both providing good registrations and drawing strong theoretical
links with Lie group theory and evolution equations in physical modeling
\cite{cotter_singular_2006,younes_shapes_2010}. 
Most often, high-dimensional voxel-wise representations are used for LDDMM
although recent interest in \emph{compact} representations
\cite{durrleman_optimal_2011,sommer_sparsity_2012} show that the number of
parameters can be much reduced. These methods use 
interpolation of the velocity field by deformation atoms to represent translational movement but 
deformation by other parts of the affine group
cannot be compactly represented.

The deformation atoms are called \emph{kernels} in LDDMM. The kernels
are centered at different spatial positions and parameters 
determine the contribution of each kernel.
In this paper, we use the partial derivative reproducing property
\cite{zhou_derivative_2008} to show that partial derivatives of kernels 
fit naturally in the LDDMM framework and 
constitute deformation atoms along with the original kernels. In particular,
these deformation atoms have a
singular \emph{higher-order momentum} and the momentum stays singular when transported by the EPDiff
evolution equations. We show how the higher-order momenta allow modeling locally affine
deformations, and they hence extend the capacity of sparsely discretized LDDMM methods. In
addition, they comprise the natural vehicle for incorporating first-order similarity 
measures in the framework. 

\subsection{Related Work}
\label{sec:related}
A number of methods for non-rigid registration have been developed during the
last decades including non-linear elastic methods \cite{pennec_riemannian_2005},
parametrizations using static velocity fields
\cite{arsigny_log-euclidean_2006,hernandez_registration_2009}, the demons
algorithm \cite{thirion_image_1998,vercauteren_diffeomorphic_2009}, and
spline-based methods \cite{rueckert_nonrigid_1999,bookstein_linear_1999}.
For the particular case of LDDMM, the groundbreaking work appeared with
the deformable template model by Grenander \cite{grenander_general_1994} and the flow approach by
Christensen et al. \cite{christensen_deformable_2002} 
together with the theoretical contributions of Dupuis et al. and
Trouv\'e \cite{dupuis_variational_1998,trouve_infinite_1995}.
Algorithms for computing optimal diffeomorphisms have been developed in
\cite{beg_computing_2005}, and \cite{vaillant_statistics_2004} uses the momentum
representation for statistics and develops a momentum based
algorithm for the landmark matching problem.

Locally affine deformations can be modeled using the Polyaffine and Log-Euclidean Polyaffine
\cite{arsigny_polyrigid_2005,arsigny_fast_2009} frameworks. The velocity of
a path of deformations is here computed using matrix logarithms, and the resulting
diffeomorphism flowed forward by integrating the velocity.
Ideas from the Polyaffine methods have recently been incorporated in e.g.
the Demons algorithm \cite{vercauteren_diffeomorphic_2009,seiler_geometry-aware_2011}. 
In LDDMM, the deformation atoms, the kernels, represent translational movement 
and the non-translational part of affine transformations cannot directly be represented.
We will show how partial derivatives of kernels constitute deformation atoms which allow 
representing the linear parts of affine transformations.
From a mathematical point of view, this is possible due to the partial derivative reproducing
property (Zhou \cite{zhou_derivative_2008}). 
The partial derivative reproducing property, partial derivatives of kernels, and
first-order momenta have previously been used in \cite{yan_cao_large_2005} to
derive variations of flow equations for LDDMM DTI registration, in
\cite{garcin_thechniques_2005} to match landmarks with vector features, and in
\cite{glaunes_transport_2005} to match surfaces with currents.
Confer the monograph \cite{younes_shapes_2010} for information on RKHSs and their role
in LDDMM.

In order to reduce the dimensionality of the
parametrization used in LDDMM, Durrleman et al. \cite{durrleman_optimal_2011} introduced a control
point formulation of the registration problem by choosing a finite set of
control points and constraining the momentum to be concentrated
as Dirac measures at the point trajectories. As we will see, higher-order
momenta make a finite control point formulation possible which is different in
important aspects. Younes \cite{younes_constrained_2011} in addition considers
evolution in constrained subspaces.

Higher-order momenta increase the capacity of the deformation parametrization, a
goal which is also treated in sparse multi-scale methods such as the kernel
bundle framework \cite{sommer_sparsity_2012}.
This method concerns the size of the kernel in contrast to the order which we deal with
here. As we will discuss in the experiments section, the size of the kernel is
important when using the higher-order representation as well, and
representations using higher-order momenta will likely complement the kernel bundle method 
if applied together.

\subsection{Content and Outline}
We start the paper with an overview of LDDMM registration and the mathematical
constructs behind the method. In the following section, we describe registration
using higher-order image information and parameterization using higher-order
momentum distributions. We then turn to the mathematical background of the
method and describe the evolution of the momentum and
velocity fields governed by the EPDiff evolution equations in the first-order
case. The next sections describes the relation to polyaffine approaches, the effect of
varying the initial conditions, and the backwards gradient transport.
We then provide examples and illustrate how the deformation represented by
first-order atoms can be interpreted when registering human brains with progressing atrophy.
The paper ends with concluding remarks and outlook. 
%The paper thus contributes by
%\begin{enumerate}[label=(\arabic*)]
%  \item using higher-order momentum distributions in the LDDMM framework to
%    represent deformation atoms that encode locally affine transformations,
%    \item showing how the order of the similarity measure approximation relates
%      to the order of the momentum,
%    \item deriving the EPDiff transport equations with first-order momenta,
%    \item computing the forward variational equations and describing the backwards gradient transport,
%    \item developing an algorithm allowing matching with higher-order information,
%    \item and demonstrating the application of the method with
%      registration examples.
%\end{enumerate}

\section{LDDMM Registration, Kernels, and Evolution Equations}
%We here give a brief introduction to LDDMM registration. For further
%information, confer the monograph \cite{younes_shapes_2010}.
%
In the LDDMM framework, registration is performed through the action of
diffeomorphisms on geometric objects. This approach is very general and allows the framework to be
applied to both landmarks, curves, surfaces, images, and tensors.  In the case of
images, the action of a diffeomorphism $\phi$ on the image
$I:\Omega\rightarrow\RR$ takes the form
$\phi.I=I\circ\phi^{-1}$, and
given a fixed image $I_f$ and moving image $I_m$, the registration amounts to a search
for $\phi$ such that $\phi.I_m\sim I_f$. In exact matching, we wish
$\phi.I_m$ be exactly equal to $I_f$ but, more frequently, we allow some amount of
inexactness to account for noise in the images and allow for smoother diffeomorphisms. This is done by
defining a similarity measure $U(\phi)=U(\phi.I_m,I_f)$ on images and a regularization measure $E_1$ 
to give a combined energy
\begin{equation}
  E(\phi)=E_1(\phi)+\lambda U(\phi.I_m,I_f)
  \label{eq:func-lddmm}
  \ .
\end{equation}
Here $\lambda$ is a positive real representing the trade-off between regularity and
goodness of fit. The similarity measure $U$ is in the simplest form the $L^2$-error 
$\int_\Omega |\phi.I_m(x)-I_f(x)|^2dx$ but more advanced measures can be used
(e.g.
\cite{roche_correlation_1998,wells_multi-modal_1996,darkner_generalized_2011}).

In order to define the regularization term $E_1$, we introduce some notations in
the following:
Let the domain $\Omega$ be a subset of $\RR^d$ with $d=2,3$, and let 
$V$ denote a Hilbert space of vector fields $v: \Omega\to \RR^d$
such that $V$ with
associated norm $\|\cdot\|_V$ is included in $L^2(\Omega,\RR^d)$ and
admissible \cite[Chap. 9]{younes_shapes_2010}, i.e. sufficiently smooth.
Given a time-dependent vector field $t\mapsto v_t$ with 
\begin{equation}
\label{eq:gv_energy}
\int_0^1\|v_t\|^2_V\,dt < \infty
\end{equation}
the associated differential equation $\partial_t\varphi_{t} = v_t\circ\varphi_{t}$ has with initial condition $\phi_s$
a diffeomorphism $\phi^v_{st}$ as unique solution at time $t$. The set $G_V$ of 
diffeomorphisms built from $V$ by such differential equations is a Lie group, and 
$V$ is its tangent space at the identity. Using the group structure,
$V$ is isomorphic to the tangent space at each point $\phi\in G_V$. The 
inner product on $V$ associated to a norm $\|\cdot\|_V$ makes $G_V$ a
Riemannian manifold with right-invariant metric. Setting $\phi^v_{00} =
\Id_\Omega$, the map $t\mapsto
\varphi^v_{0t}$ is a path from $\Id_\Omega$ to $\phi$ with energy given by
\eqref{eq:gv_energy} and generated by $v_t$. We will use this notation
extensively in the following. A critical path for the energy \eqref{eq:gv_energy}
is a geodesic on $G_V$, and the regularization term $E_1$ is defined using the
energy by
\begin{equation}
  E_1(\phi)
  =
  \min_{v_t\in V,\phi^v_{01}=\phi}\int_0^1\left\|v_s\right\|_V^2ds
  \ ,
  \label{eq:reg-lddmm}
\end{equation}
i.e. it measures the minimal energy of diffeomorphism paths from $\Id_\Omega$
to $\phi$. Since the energy is high for paths with great variation, the term 
penalizes highly varying paths, and a low value of $E_1(\phi)$ thus implies that
$\phi$ is regular. 

\subsection{Kernel and Momentum}
\label{sec:kermom}
As a consequence of the assumed admissibility of $V$, the evaluation
functionals $\delta_x:v\mapsto v(x)\in \RR^d$ is well-defined and continuous for any $x\in \Omega$. Thus, for any 
$z\in\RR^d$ the map $z\otimes\delta_x: v\mapsto z^Tv(x)$ belongs to the topological
dual $V^*$, i.e. the continuous linear maps on $V$. This in turn implies 
the existence of spatially dependent matrices $K:\Omega\times\Omega\to
\RR^{d\times d}$, the \emph{kernel}, such that, for
any constant vector $z\in\RR^d$, the vector field $K(\cdot,x)z\in V$ represents
$z\otimes\delta_x$ and $\ip{K(\cdot,x)z,v}_V=z\otimes \delta_x(v)$ for any $v\in V$, 
point $x\in\Omega$ and vector $z\in\RR^d$. This latter property is denoted the
reproducing property and gives $V$ the structure of a reproducing kernel Hilbert space
(RKHS). Tightly connected to the norm and kernels is the notion of \emph{momentum} given
by the linear momentum operator $L:V\to V^*\subset L^2(\Omega,\RR^d)$ which satisfies
\begin{equation*}
  \ip{Lv,w}_{L^2(\Omega,\RR^d)}
 =
 \int_\Omega \big(Lv(x)\big)^Tw(x)dx
 =
 \ip{v,w}_V
\end{equation*}
for all $v,w\in V$.
The momentum operator connects the inner product on $V$ with the inner product in
$L^2(\Omega,\RR^d)$, and the image $Lv$ of an element $v\in V$ is denoted the
momentum of $v$.
The momentum $Lv$ might be singular and in fact $L\big(K(\cdot,y)z\big)(x)$ is
the Dirac measure $\delta_y(x)z$. Considering $K$ as the map $z\mapsto
\int_\Omega K(\cdot,x)z(x)dx$, $L$ can be viewed as the inverse of $K$.
We will use the symbol $\rho$ for the momentum when considered as a functional in $V^*$
while we switch to the symbol $z$ when the momentum is realized as a vector field on $\Omega$ or
for the parameters when the momentum consists of a finite number of singular point measures.

Instead of deriving the kernel from $V$, the opposite approach can be used: build
$V$ from a kernel, and hence impose the regularization in the framework from the
kernel. With this approach, the kernel is often chosen to ensure rotational and 
translational invariance \cite{younes_shapes_2010} and the scalar Gaussian kernel
$K(x,y)=\exp(-\frac{\|x-y\|^2}{\sigma^2})\Id_d$
is an often used choice. Confer \cite{fasshauer_reproducing_2011} for details on
the construction of $V$ from Gaussian kernels.

\subsection{Optimal Paths: The EPDiff Evolution Equations}
\newcommand{\Ad}{\mathrm{Ad}}%
The relation between norm and momentum leads to convenient
equations for minimizers of the energy \eqref{eq:func-lddmm}. In particular,
the EPDiff equations for the evolution of the momentum $z_t$ for optimal paths assert 
that if $\phi_t$ is a path minimizing $E_1(\phi)$ with $\phi_1=\phi$ minimizing
$E(\phi)$ and $v_t$ is the derivative of $\phi_t$ then $v_t$ satisfies the
system
\begin{align*}
  &v_t=\int_\Omega K(\cdot,x)z_t(x)dx\ ,
  \\&
%\quad
  \frac{d}{dt}z_t=-Dz_tv_t-z_t\nabla\cdot v_t-(Dv_t)^Tz_t
\end{align*}
with $Dz_t$ and $Dv_t$ denoting spatial differentiation of the momentum and
velocity fields, respectively.
The first equation connects the momentum $z_t$ with the velocity $v_t$, and the
second equation describes the time evolution of the momentum. In the
most general form, the EPDiff equations describe the evolution of the momentum
using the adjoint map. Following \cite{younes_shapes_2010}, define 
the adjoint $\Ad_\phi v(x)=(D\phi\,v)\circ \phi^{-1}(x)$ for $v\in V$. The dual of the
adjoint is the functional
$\Ad_\phi^*$ on the dual $V^*$ of $V$ defined by
$(\Ad_\phi^*\rho|v)=(\rho|\Ad_\phi(v))$.\footnote{
Here and in the following, we will use the notation $(p|v):=p(v)$ for evaluation
of the functional $p\in V^*$ on the vector field $v\in V$.
} Define in addition
$\Ad_\phi^Tv=K(\Ad_\phi^*(Lv))$ which then satisfies
$\ip{\Ad_\phi^Tv,w}=(\Ad_\phi^*(Lv)|w)$,
and let $\nabla_{\phi} U$ denote the gradient of the similarity measure $U$ with
respect to the inner product on $V$ so that
$\ip{\nabla_{\phi}U,v}_V=\partial_\epsilon U(\psi_{0\epsilon}^v\circ\phi)$ for
any variation $v\in V$ and diffeomorphism path $\psi_{0\epsilon}^v$ with
derivative $v$.
For optimal paths $v_t$, the EPDiff equations assert that
$v_t=\Ad_{\phi_{t1}^v}^Tv_1$ with $v_1=-\frac{1}{2}\nabla_{\phi_{01}^{v}} U$
which leads to the conservation of momentum property for optimal paths.
Conversely, the EPDiff equations reduce to
simpler forms for certain objects. For landmarks $x_1,\ldots,x_N$, the
momentum will be concentrated at point trajectories $x_{t,k}:=\phi_t(x_k)$ as Dirac measures 
$z_{t,k}\otimes\delta_{x_{t,k}}$ leading to the finite dimensional system of ODE's
\begin{equation}
  \begin{split}
    & v_t=\sum_{l=1}^NK(\cdot,x_{t,l}))z_{t,l}\ ,
    \quad\frac{d}{dt}\phi_t(x_k)=v_t(x_{t,k})\ ,\\
    &\frac{d}{dt}z_{t,k}=-\sum_{l=1}^N\nabla_1K(x_{t,l},x_{t,k})z_{t,k}^Tz_{t,l}
    \ .
  \end{split}
  \label{sys:point-epdiff}
\end{equation}

\section{Registration with Higher-Order Information}
We here introduce higher-order momentum distributions for registration using
higher-order information with the LDDMM framework. We start by
motivating the construction by considering the approximation used when computing the similarity
measure. We then describe the deformation encoded by higher-order momenta and the evolution equations 
in the finite case, and we use this to derive a registration algorithm using first-order information. The
mathematical background behind the method will be presented in the following
sections.\ \\

We will motivate the introduction of higher-order momenta by considering
a specific case of image registration:
we take on the goal of using a control point formulation \cite{durrleman_optimal_2011} 
when solving the registration problem \eqref{eq:func-lddmm} and hence aim for
using a relatively sparse sampling of the velocity or momentum field. To achieve this, we will consider
the coupling between the transported control points $\{\phi^{-1}(x_1),\ldots,\phi^{-1}(x_N)\}$
and the similarity measure in order to ensure the momentum stays singular
and localized at the point trajectories while removing the need for warping the
entire image at every iteration of the optimization process. 

Considering a similarity measure $U(\phi)=\int_\Omega F(I_m(\phi^{-1}(x)),I_f(x))dx$
as discussed in the introduction, and a finite discretization
$\tilde{U}^0(\phi)=1/N\sum_{k=1}^NF(\phi.I_m(x_k),I_f(x_k))$ 
with a sparse set of control points $\{x_k\}$. While using $\tilde{U}^0(\phi)$
to drive registration of the images will be very 
efficient in evaluating the warp in few points, it will suffer correspondingly from only
using image information present in those points. Apart from not being robust
under the presence of noise in the images, the discretization implies that
local dilation or rotation around the points $\phi^{-1}(x_k)$ cannot be detected: any variation
$v\in V$ of $\phi$ keeping $\phi^{-1}(x_k)$ constant for all $k=1,\ldots,N$ will not change
$\tilde{U}^0(\phi)$. Formally, if $\psi_{0\epsilon}$ is a diffeomorphism path
that is equal to $\phi$ at $t=0$ and has derivative $v$ at $t=0$, i.e. 
$\partial_\epsilon\psi_{0\epsilon}=v$ and $\psi_{00}=\phi$, then
\begin{align*}
  &\partial_\epsilon F( \psi_{0\epsilon}.I_m(x_k),I_f(x_k))
=
\partial_1
F(\phi.I_m(x_k),I_f(x_k))
\cdot
\big(\nabla_{\phi^{-1}(x_k)}I_m\big)^T v(\phi^{-1}(x_k))
\end{align*}
which vanishes if $v(\phi^{-1}(x_k))=0$. Here $\partial_1 F$ denotes the
derivative of $F:\RR^2\rightarrow\RR$ with respect to the first variable.

A simple way to include more image information in the similarity measure 
is to convolve with a kernel $K_s$, and thus extend $\tilde{U}^0$ to
\begin{equation*}
  U^1(\phi)
  =
  \frac{1}{N}\sum_{k=1}^Nc_{K_s}\int_\Omega K_s(p+x_k,x_k)F(\phi.I_m(p+x_k),I_f(p+x_k))dp
\end{equation*}
with $c_{K_s}$ a normalization constant.
If $K_s$ is a box kernel, this amounts to a finer sampling of both the image
and warp, and hence a finer discretization of the Riemann integral.
The kernel $K_s$ should not be confused with the RKHS kernel connected to the
norm on $V$ that is used when generating the $V$-gradient. A Gaussian kernel may be
used for $K_s$,
and more information on
using smoothing kernels for intensity based image registration can be found
in \cite{darkner_generalized_2011,xiahai_zhuang_nonrigid_2011}.

The measure
$U^1(\phi)$ is problematic since a variation of $\phi$ would affect not only the
point $\phi^{-1}(x_k)$ but also $\phi.I_m(p+x_k)$, and $U^1(\phi)$ will therefore be
dependent on $\phi.I_m(p+x_k)$ for any $p$ where $K_s(p,x_k)$ is non-zero.
In this situation, the momentum is no longer concentrated in Dirac measures
located at $\phi_t^{-1}(x_k)$, and it will be necessary to increase the sampling
of the warp.
However, a first-order expansion of $\phi^{-1}$ yields the approximation
\begin{equation}
  \tilde{U}^1(\phi)
  =
  \frac{1}{N}\sum_{k=1}^Nc_{K_s}\int_\Omega
  K_s(p+x_k,x_k)F(I_m(D_{x_k}\phi^{-1}p+\phi^{-1}(x_k)),I_f(p+x_k))dp
  \ .
  \label{eq:tildeU1}
\end{equation}
The measure $\tilde{U}^1(\phi)$ is now again local depending only on
$\phi^{-1}(x_k)$ and the 
first-order derivatives $D_{x_k}\phi^{-1}$. It offers the stability provided by
the convolution with $K_s$, and, importantly, variations $v$ of $\phi$ keeping
$\phi^{-1}(x_k)$ constant but
changing $D_{x_k}\phi^{-1}$ do indeed affect the similarity measure. This implies
that $\tilde{U}^1(\phi)$ is able to catch rotations and dilations and drive the
search for optimal $\phi$ accordingly. Please note the differences with the approach of
Durrleman et al. \cite{durrleman_optimal_2011}: 
when using $\tilde{U}^1(\phi)$ as outlined here, the need for flowing the entire moving 
image forward is removed and the momentum field will stay singular
\emph{directly} thus removing the need for constraining the form of the velocity field.

\subsection{Evolution and Deformation with Higher-Order Information}
The dependence on $D\phi$ in the similarity measure $\tilde{U}^1(\phi)$
raises the question of how to represent variations of $D\phi$ in the LDDMM
framework. As we will outline here, higher-order momenta appear as the natural
choice for such a representation that keeps the benefits of
the finite control point formulation. Mathematical details will follow in the
next sections.

Recall the reproducing property of the RKHS structure, i.e. 
$\ip{K(\cdot,x)z,v}_V=z\otimes\delta_x(v)$ for $v\in V$, $x\in\Omega$ and
$z\in\RR^d$. Let us define the maps $z\otimes D_x^\alpha:V\rightarrow\RR$
that extend the Diracs $z\otimes\delta_x(v)$ by measuring the \emph{derivative}
of $v$ at $x$. These will be denoted \emph{higher-order Diracs}, and we say that
the momentum distribution is of higher-order if it is a sum of higher order
Diracs. When applying the momentum operator $L$ to the higher-order Diracs, we will 
get partial derivatives $D_x^\alpha K$ of the RHKS kernel $K$.

In particular, we will see that when
using similarity measures such as $\tilde{U}^1(\phi)$, the momentum field will
be \emph{a linear combination of higher-order Diracs} and the velocity field will,
correspondingly, be \emph{a linear combination of partial derivatives of $K$}.
This will imply that the finite dimensional system of ODE's
\eqref{sys:point-epdiff} describing the EPDiff equations in the landmark case 
will be extended so that the velocity $v_t$ will
contain partial derivatives $D_x^\alpha K$. In the first-order case, we will get the velocity
\begin{equation}
  v(\cdot)=\sum_{l=1}^N\big(K(\cdot,x_l)z_l+\sum_{j=1}^dD^jK(\cdot,x_l)z_l^j\big)
\end{equation}
where $z_i$ denotes the coefficients of the Dirac measures as in \eqref{sys:point-epdiff}
but now the additional vectors $z_i^j$ denote the coefficients of the
first-order Diracs $z_i^j\otimes D_{x_i}^j$ for each of the $d$ dimensions
$j=1,\ldots,d$. We will later show how these
coefficients evolve. Combined with knowledge of how variations of $z_i$ and
$z_i^j$ affect the system, we can transport variational information along the
optimal paths specified by the EPDiff equations and thus provide the necessary
building blocks for a first-order registration algorithm.

Figure~\ref{fig:kernels} illustrates how the local translation encoded by the
kernel is complemented by locally affine deformation when
incorporating first-order momenta and corresponding partial derivatives of the kernel. 
Using the language of deformation
atoms, the first-order constructions adds partial derivatives of kernels to the usual set of atoms, and
the deformation atoms are thus able to compactly encode expansion,
contraction, rotation etc. We can directly \emph{interpret} the coefficients of the
first-order momenta as controlling the magnitude of these first-order deformations.
In Figure~\ref{fig:shots} in the experiments section,
we give additional illustrations of the deformation encoded by the new atoms.
\begin{figure}[t]
\begin{center}
  \parbox{0.99\columnwidth}{
\begin{center}
  \subfigure[The RHKS kernel encodes local
  translation.]{\includegraphics[width=0.30\columnwidth,trim=40 35 40 20,clip=true]{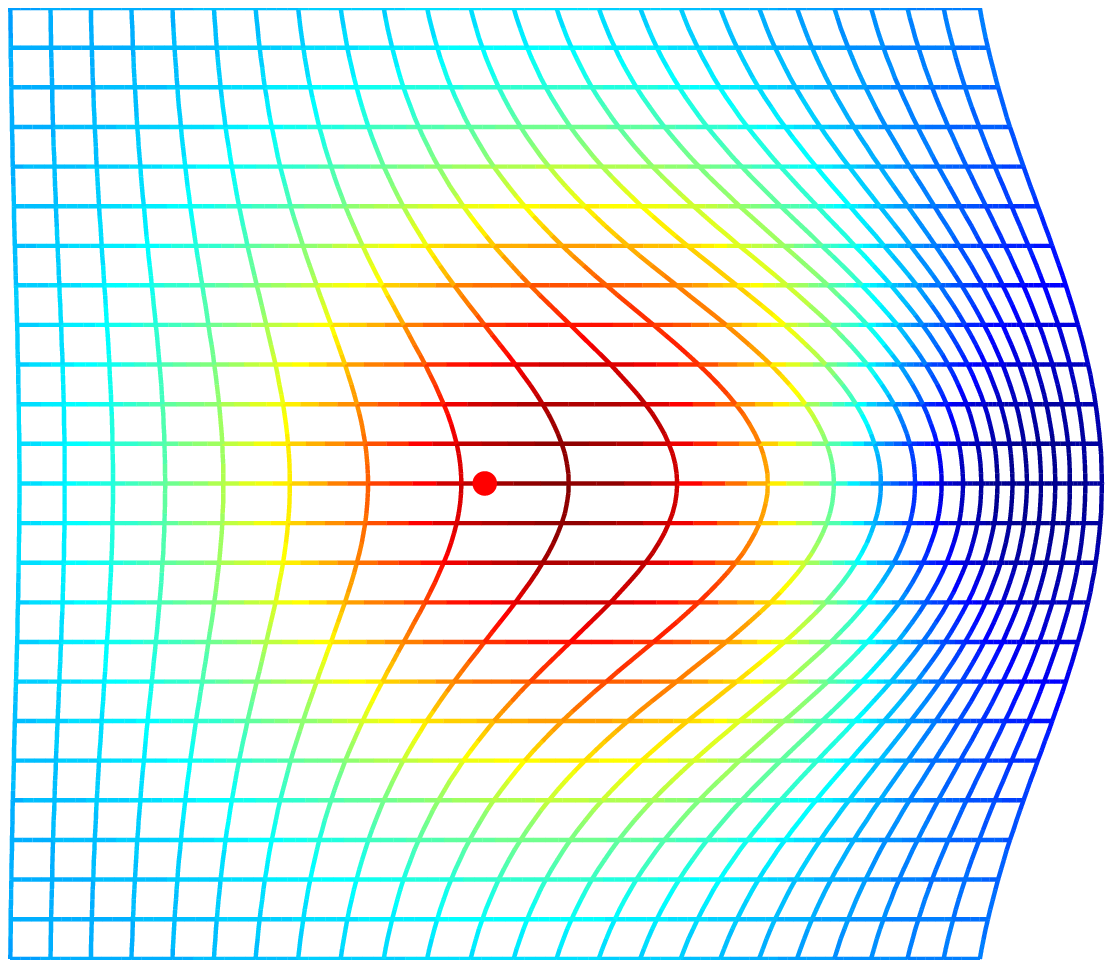}} 
  \hspace{0.6em}
  \subfigure[Ensembles of kernels can approximate locally affine deformation.]{\includegraphics[width=0.30\columnwidth,trim=40 35 40 20,clip=true]{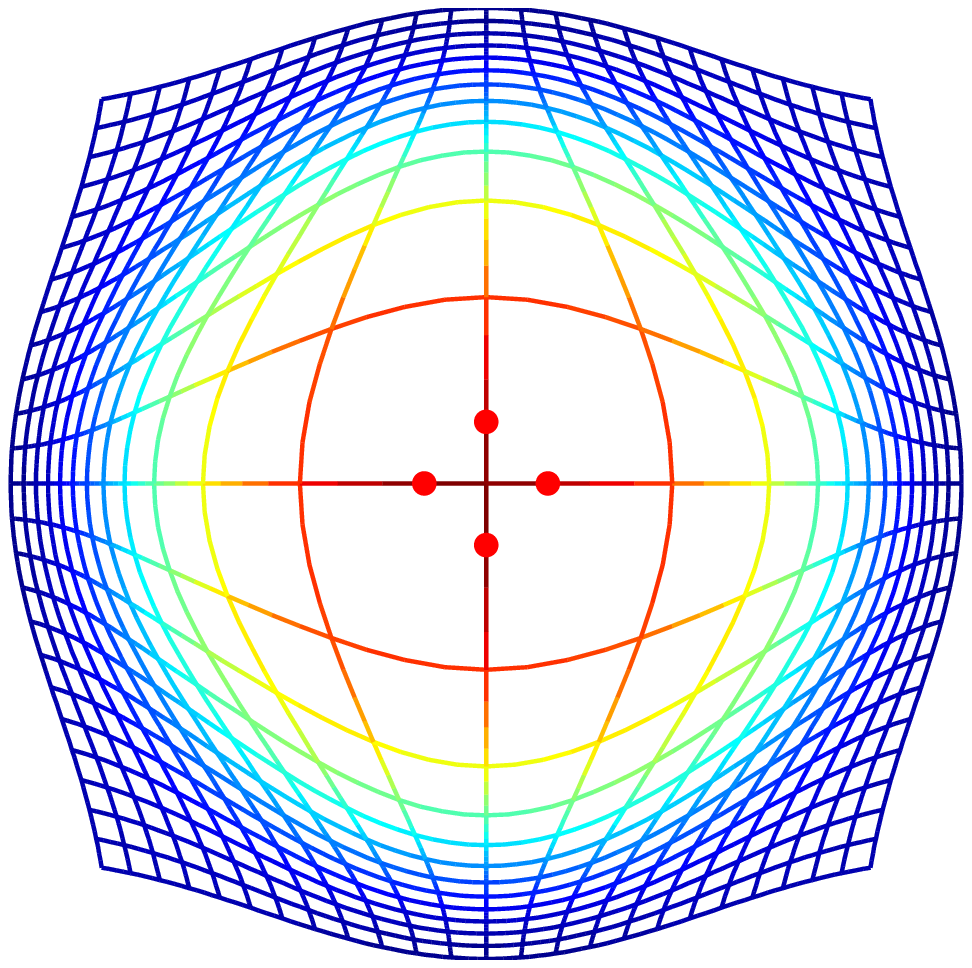}}
  \hspace{0.6em}
  \subfigure[Derivatives of the kernel directly encode locally affine deformation.]{\includegraphics[width=0.30\columnwidth,trim=40 35 40 20,clip=true]{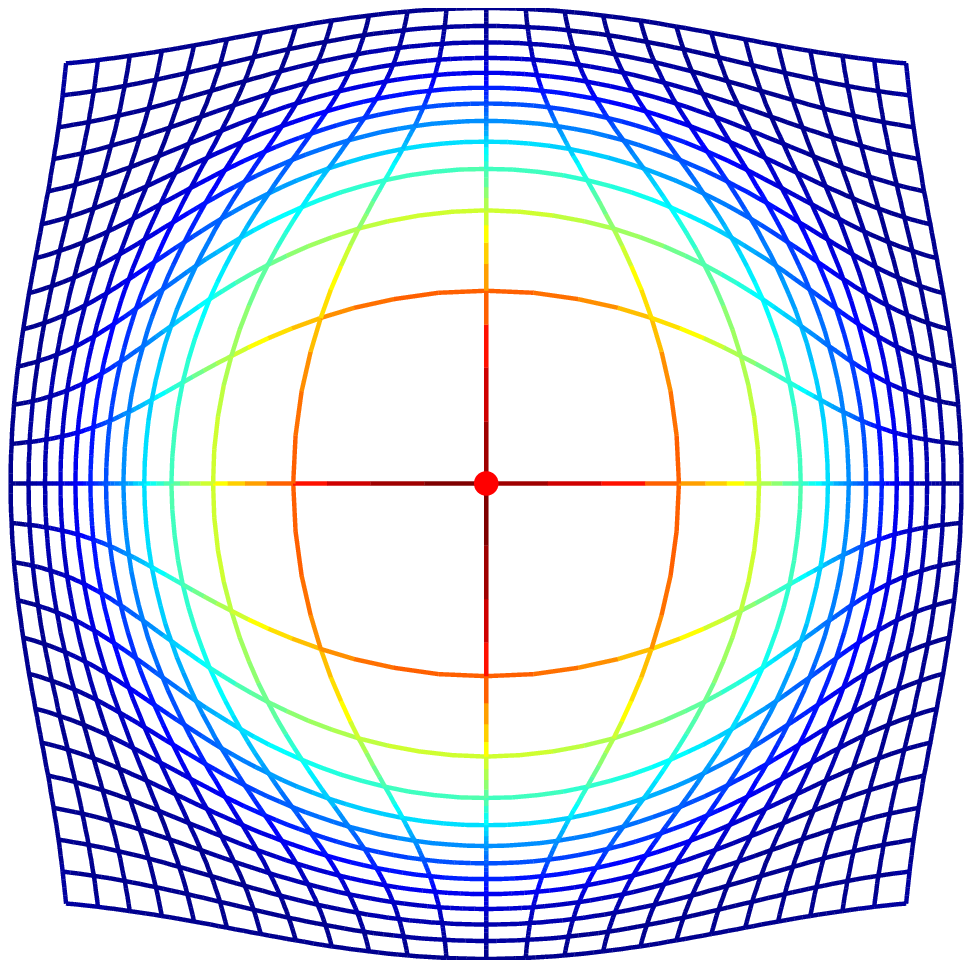}}
\end{center}
  }
\end{center}
\caption{
The deformation encoded by the kernel: (a) the RHKS kernel, here a Gaussian of 
scale 8 in grid units, encodes local translation; (b) locally affine
deformation, here expansion, can be approximated by placing kernels close
together. When moving these kernels infinitesimally close, the 
derivative of the kernel arises in the limit, and (c) the derivative encode locally
affine deformation directly.
With higher-order momenta, we will use derivatives of the kernel as
deformation atoms.
}
\label{fig:kernels}
\end{figure}

\subsection{Algorithm for First-Order Registration}
In this section, we will derive a registration algorithm for similarity measures
incorporating first-order information such as $\tilde{U}^1(\phi)$. Since the
algorithm works for general first-order measures, we will again let $U$
denote the similarity measure with $\tilde{U}^1(\phi)$ being just a particular
example.

There exists various choices of optimization algorithms for LDDMM registration.
Roughly, they can be divided into two groups based on whether they represent the initial
momentum/velocity or the entire path $\phi_t$. Here, we take the approach of
incorporating first-order momenta with the shooting method of e.g. Vaillant et al. 
\cite{vaillant_statistics_2004}. The algorithm will take a guess for the initial
momentum, integrate the EPDiff equations forward, compute the similarity measure
gradient $\nabla U(\phi)$, and flow the gradient backwards to provide an improved guess.

The registration problem \eqref{eq:func-lddmm} consists of both the similarity
measure and the minimal path energy $E_1$. For e.g. landmark based
registration, the similarity $U(\phi)$ is most often expressed in terms of $\phi$ directly whether as
the similarity measure is usually dependent on the inverse $\phi^{-1}$ for image registration.
In the first case, the gradient $\nabla_\phi U$ is known, and,
given the initial momentum $z_0$, we can obtain the gradient $\nabla_{z_0}U$ 
for a gradient descent based
optimisation procedure from the backwards transport equations
that we derive in Section~\ref{sec:variations}. For the
energy part, it is a fundamental property of critical paths in the LDDMM
framework that the energy stays constant along the path. Thus, 
we can easily compute the
gradient from the expressions provided in Section~\ref{sec:hom}. Given this, the 
zeroth order matching algorithm in the
initial momentum is generalized to zeroth and first-order momenta in
Algorithm~\ref{alg:alg1}.
\begin{algorithm}
\begin{algorithmic}
  \State $z_0\gets \mbox{initial guess for initial momentum}$
  \Repeat
  \State Solve EPDiff equations forward
  \State Compute similarity $U$
  \State Solve backwards the transpose equations
  \State Compute the energy gradient $\nabla\|v_0\|^2$
  \State Update $z_0$ from $\nabla\|v_0\|^2+\nabla_{z_0} U$
  \Until{convergence}
\end{algorithmic}
\caption{Matching with Zeroth and First-Order Momenta.}
\label{alg:alg1}
\end{algorithm}

Traditionally, the similarity measure $U(\phi)$ is in image matching formulated
using the inverse of $\phi$, and this approach was taken when
formulating the approximation $\tilde{U}^1(\phi)$ in \eqref{eq:tildeU1}. For this reason, at finite
control point formulation is naturally expressed using a sampling
$\{x_1,\ldots,x_N\}$ in the \emph{target} image with the algorithm optimizing for the momentum
$z_1$ at time $t=1$. The evaluation points $\phi^{-1}(x_k)$ are then
generated by flowing \emph{backwards} from $t=1$ to $t=0$, and the gradient of $U(\phi)$ 
can be computed in $\phi^{-1}(x_k)$ before being flowed \emph{forwards} to update
$z_1$. 
Algorithm~\ref{alg:alg1} will
accommodate this situation by just reversing the integration directions. The control points can be
chosen either at e.g. anatomically important locations, at random, or on a regular grid.
In the experiments, we will register expanding ventricles using control points placed
in the ventricles.

\subsection{Numerical Integration}
The integration of the flow equations can be performed with standard Runge-Kutta
integrators such as Matlabs \texttt{ode45} procedure.
With zeroth order momenta only and $N$ points, the forward and backwards system consist of $2dN$
equations. With zeroth and first-order momenta, the forward system is extended to
$N(2d+d^2)$ and the backwards system to $2N(d+d^2)$. For $d=3$, this implies
an $2.5$ times increase in the size of the forward system and $4$ times increase
in the backwards system. As suggested in Figure~\ref{fig:kernels}, the
first-order system can be approximated using ensembles of zeroth-order atoms.
Such an approximation would for $d=3$ require at least four zeroth-order atoms
for each first-order atom making the size of the approximating system
equivalent to the first-order system. Due to the non-linearity of the systems, 
the effect of the approximation introduced with such as an approach is not
presently established.

In addition to the increase in the size of the systems,
the extra floating point operations necessary for computing the
more complicated evolution equations should be considered. The additional computational effort should, however, be
viewed against the fact that the finite dimensional system can
contain orders of magnitude fewer control points,
and the added capacity of deformation description included in the
derivative information. In addition and in contrast to previous approaches, 
we transport the similarity gradient \emph{only} at the control point
trajectories, again an order of magnitude reduction of transported information.

\section{Higher-Order Momentum Distributions}
\label{sec:hom}
We now link partial derivatives of kernels to higher-order momenta
using the derivative reproducing property, and we provide details on the EPDiff
evolution equations that we outlined in the previous section. We underline that
the analytical of structure of LDDMM is not changed when incorporating
higher-order momenta, and the evolution equations will thus be particular
instances of the general EPDiff equations.  These equations in Hamiltonian form
constitutes the explicit expressions that allows implementation of the
registration algorithm.

We will restrict to scalar kernels when appropriate for simplifying the notation. 
Scalar kernels are diagonal matrices
where all diagonal elements are equal. Thus, we can consider $K(x,y)$ both a
matrix and a scalar so that
the entries $K^j_i(x,y)$ of the kernel in matrix form equals the scalar $K(x,y)$ if 
and only if $i=j$ and $0$ otherwise.

\subsection{Derivative Reproducing Property}
Recall the reproducing property of the RKHS structure, i.e. 
$\ip{K(\cdot,x)z,v}_V=z\otimes\delta_x(v)$ for $v\in V$, $x\in\Omega$ and
$z\in\RR^d$. Zhou \cite{zhou_derivative_2008} shows that this property holds not
only for the kernel but also for its partial derivatives. Letting 
$D_x^\alpha v$ denote the derivative of $v\in V$ at $x\in\Omega$ with respect to the multi-index $\alpha$,
\begin{equation*}
  D_x^\alpha v
  =
  \frac{\partial^{|\alpha|}}{\partial_{x^1}^{\alpha_1}\ldots\partial_{x^d}^{\alpha_d}}
  v(x)
\end{equation*}
and defining $(D_x^\alpha Kz)(y)=D_x^\alpha (K(\cdot,y)z)$ for $z\in\RR^d$, Zhou proves
that $D_x^\alpha Kz\in V$ and that the \emph{partial derivative reproducing property}
\begin{equation}
  \ip{D_x^\alpha Kz,v}_V=z^TD_x^\alpha(v)
  \label{eq:zhou}
\end{equation}
holds when the maps in $V$ are sufficiently smooth for the derivatives to exist.
We denote the maps $z\otimes D_x^\alpha:V\rightarrow\RR$
defined by $z\otimes D_x^\alpha(v):=z^T D_x^\alpha v$ \emph{higher-order
Diracs}, and we say that the momentum distribution is of higher order if it is a
sum of higher-order Diracs. It follows that
\begin{equation*}
  z\otimes D_x^\alpha
  =
  \big(
  v \mapsto
  \ip{D_x^\alpha Kz,v}_V
  \big)
  \in
  V^*
  \ .
\end{equation*}
As a consequence, we can connect higher-order momenta
and partial derivatives $D_x^\alpha K$ of the kernel. Recall that the momentum map $L:V\rightarrow V^*$ satisfies 
$\ip{Lv,w}_{L^2}=\ip{v,w}_V$. With the higher-order momenta,
\begin{equation*}
  \ip{L D_x^\alpha Kz,v}_{L_2}
  =\ip{D_x^\alpha Kz,v}_V
  =z\otimes D_x^\alpha(v)
  =\ip{z\otimes D_x^\alpha,v}_{L^2}
  \ .
\end{equation*}
Thus $L D_x^\alpha Kz=z\otimes D_x^\alpha$ or, shorter, $L D_x^\alpha K=D_x^\alpha$.
That is, partial derivatives of the kernel and higher-order momenta corresponds just as
the kernels and Diracs measures in the usual RKHS sense.

Consider a map on diffeomorphisms $U:G_V\rightarrow\RR$ e.g.
an image similarity measure dependent on $\phi$. In a finite
dimensional setting with $N$ evaluation points $x_k$, $U$ would decompose as 
$U(\phi)=P\circ Q(\phi)$ with $Q(\phi)=(\phi(x_1),\ldots,\phi(x_N))$ and $P:\RR^{dN}\rightarrow\RR$.
Introducing higher-order momenta, we let
$Q(\phi)=(D_{x_1}^{\alpha_1}(\phi),\ldots,D_{x_N}^{\alpha_J}(\phi))$ with $J$
multi-indices $\alpha_j$, and decompose $U$ as $U(\phi)=P\circ Q(\phi)$ with
$P:\RR^{dNJ}\rightarrow\RR$. We allow $\alpha_j$ to be empty and hence
incorporate the standard zeroth order case. 
The partial derivative reproducing property now lets us compute the $V$-gradient
of $U$ as a sum of partial derivatives of the kernel.
\begin{proposition}
  Let $\nabla^{kj}P$ denote the gradient with respect to the variable indexed by 
  $D_{x_k}^{\alpha_j}(\phi)$ in the expression for $Q$. Then the gradient
  $\nabla_\phi U\in V$ of $U$ with respect to the inner product in $V$ is given by
  $
    \nabla_\phi U
    =
    \sum_{k=1}^N\sum_{j=1}^J
    D_{x_k}^{\alpha_j}K\nabla_{Q(\phi)}^{kj}P
  $.

  \label{prop:grad}
\end{proposition}
\begin{proof}
  The gradient $\nabla_\phi U$ at $\phi$ is defined by $\ip{\nabla_\phi U, v}=\partial_\epsilon
  U(\epsilon v+\phi)$ for all variations $v\in V$. For such $v$, we get using
  \eqref{eq:zhou} that
  \begin{align*}
    \partial_\epsilon U(\epsilon v+\phi)
    &=
    \partial_\epsilon P\circ Q(\epsilon v+\phi)
    =
    \partial_\epsilon P(D_{x_k}^{\alpha_j}(\epsilon v+\phi))
    =
    \partial_\epsilon P(\epsilon D_{x_k}^{\alpha_j}v+D_{x_k}^{\alpha_j}\phi)
    \\
    &=
    \sum_{k=1}^N\sum_{j=1}^J(\nabla_{Q(\phi)}^{kj} P)^TD_{x_k}^{\alpha_j}v
    =
    \ip{\sum_{k=1}^N\sum_{j=1}^JD_{x_k}^{\alpha_j}\nabla_{Q(\phi)}^{kj} P,v}_V
    \ .
  \end{align*}
\end{proof}

\subsection{Momentum and Energy}
As a result of Proposition~\ref{prop:grad}, the momentum of the gradient of $U$
is
  $
    L\nabla_\phi U
    =
    \sum_{k=1}^N\sum_{j=1}^J
    \nabla_{Q(\phi)}^{kj}P\otimes D_{x_k}^{\alpha_j}
  $.
In general, if $v\in V$ is represented by a sum of higher-order momenta, the energy
$\|v\|_V^2$ can be computed using \eqref{eq:zhou} as a sum of partial
derivatives of the kernel evaluated at the points $x_k$. To keep the notation brief, we restrict
to sums of zeroth and first-order momenta in the following. 
If 
$
  v(\cdot)=\sum_{k=1}^N\big(K(x_k,\cdot)z_k+\sum_{j=1}^dD^jK(x_k,\cdot)z_k^j\big)
  \ ,
$
we get the energy
\begin{equation}
\begin{split}
  \|v\|_V^2
  &=
  \ip{
  \sum_{k=1}^N\big(K(x_k,\cdot)z_k+\sum_{j=1}^dD^jK(x_k,\cdot)z_k^j\big)
  ,
  \sum_{k=1}^N\big(K(x_k,\cdot)z_k+\sum_{j=1}^dD^jK(x_k,\cdot)z_k^j\big)
  }_V
  \\
  &=
  \sum_{k,l=1}^N
  \ip{
  K(x_l,\cdot)z_l
  ,
  K(x_k,\cdot)z_k
  }_V
  +
  \sum_{k,l=1}^N\sum_{j,i=1}^d
  \ip{
  D^jK(x_l,\cdot)z_l^j
  ,
  D^{i}K(x_k,\cdot)z_k^{i}
  }_V
  \\
  &\qquad
  +
  2\sum_{k,l=1}^N\sum_{j=1}^d
  \ip{
  D^jK(x_l,\cdot)z_l^j
  ,
  K(x_k,\cdot)z_k
  }_V
  \\
  &=
  \sum_{k,l=1}^N
  \big(
  z_l^TK(x_l,x_k)z_k
  +
  \sum_{j,i=1}^d
  z_k^{i,T}D_2^{i}D_1^jK(x_l,x_k)z_l^j
  +
  2\sum_{j=1}^d
  z_k^TD_1^jK(x_l,x_k)z_l^j
  \big)
\end{split}
\label{eq:energy}
\end{equation}
with $D_q^j K(\cdot,\cdot)$ denoting differentiation with the respect to the $q$th
variable, $q=1,2$, and $j$th coordinate, $j=1,\ldots,d$.
For scalar symmetric kernels, this expression reduces to
\begin{align*}
  \|v\|_V^2
  &=
  \sum_{k,l=1}^N
  \big(
  z_l^TK(x_l,x_k)z_k
  +
  \sum_{j,i=1}^d
  \big(D_2\nabla_1 K(x_l,x_k)\big)^{i}_jz_k^{i,T}z_l^j
  \\
  &\quad\ 
  +
  2\sum_{j=1}^d
  (\nabla_1 K(x_l,x_l))^jz_k^Tz_l^j
  \big)
   \ .
\end{align*}

\subsection{EPDiff Equations}
It is important to note that higher-order momenta offer a convenient
representation for the gradients of maps $U$ incorporating derivative
information but since the partial derivatives of kernels are members of $V$ and
the higher order momentum in
the dual $V^*$, the analytical of structure of LDDMM is not changed. In particular, the
adjoint form of the EPDiff equations, i.e. that optimal paths $v_t$
satisfy $v_t=\Ad_{\phi_{t1}^v}^Tv_1$ with
$v_1=-\frac{1}{2}\nabla_{\phi_{01}^{v}} U$,
is still valid. The momentum $\rho_1=Lv_1$ is transported to the
momentum $\rho_t$ by $\Ad_{\phi_{t1}^v}^*p_1$. Because
\begin{align*}
  (\rho_t|w)
  =
  (\rho_1|\Ad_{\phi_{t1}^v}(w))
  =
  (\rho_1|(D\phi_{t1}^v\,w)\circ (\phi_{t1}^v)^{-1})
  \ ,
\end{align*}
if $\rho_1$ is a sum of higher-order Diracs, $\rho_t$ will be sum of higher-order 
Diracs for all $t$. However, since the time evolution of $(\rho_t|w)$ with the above
relation involves derivatives of $D\phi_{t1}^v$, 
this form is inconvenient for computing $\rho_t$. Instead, we
make use of the Hamiltonian form of the EPDiff equations \cite[P.
265]{younes_shapes_2010}. Here, the momentum $\rho_t$ is pulled back to
$\rho_0$ but with a coordinate change of the evaluation vector field:
the Hamiltonian form $\mu_t$ is defined by 
$\apply{\mu_t}{w}:=\apply{\rho_0}{(D\phi_{0t}^v)^{-1}(y)w(y)}_y$
where the subscript stresses that $(D\phi_{0t}^v)^{-1}(y)w(y)$ is evaluated as a
$y$-dependent vector field. To simplify the notation, we write just $\phi_t$
instead of $\phi_{0t}^v$. Using this notation,
the evolution equations become
\begin{equation}
\begin{split}
  &\partial_t\phi_{t}(y)
  =
  \sum_{j=1}^d\apply{\mu_t}{K^j(\phi_{t}(x),\phi_{t}(y))}_xe_j
  \\
  &
  \apply{\partial_t\mu_t}{w}
  =
  -\sum_{j=1}^d\apply{\mu_t}{\apply{\mu_t}{D_2K^j(\phi_{t}(x),\phi_{t}(y))w(y)}_xe_j}_y
  \ .
  \label{sys:hamilton}
\end{split}
\end{equation}
The system forms an ordinary differential equation describing the evolution of the
path and momentum \cite{younes_shapes_2010} when $(\rho_0|w)$ does not involve
derivatives of $w$, e.g. when $\rho_0$ and hence $\rho_t$ is a vector field
$z_t$ and the first equation therefore is an integral
\begin{equation*}
  \partial_t\phi_{t}(y)
  =
  \int_{\Omega}K(\phi_{t}(y),\phi_{t}(x))z_t(x)dx
  \ .
\end{equation*}
For the higher-order case, we will
need to incorporate additional information in the system.

Again we restrict to finite sums of zeroth and first-order point measures, and we 
therefore work with initial momenta on the form
\begin{equation}
  \rho_0
  =
  \sum_{k=1}^N
  z_{0,k}\otimes\delta_{x_{0,k}}
  +
  \sum_{k=1}^N
  \sum_{j=1}^d
  z_{0,k}^j\otimes D^j\delta_{x_{0,k}}
  \label{eq:rho0}
\end{equation}
with $x_{t,r}$ as usual denoting the point positions $\phi_{t}(x_i)$ at time
$t$. Then
\begin{align*}
  \apply{\mu_t}{w}
  &=
  \apply{\rho_0}{D\phi_{t}(y)^{-1}w(y)}_y
  \\
  &=
  \int_\Omega 
  \Big(
  \sum_{k=1}^N
  z_{0,k}\otimes\delta_{x_{0,k}}
  +
  \sum_{k=1}^N
  \sum_{j=1}^d
  z_{0,k}^j\otimes D^j\delta_{x_{0,k}}
  \Big)
  D\phi_{t}(y)^{-1}w(y)
  dy
  \\
%  &=
%  \sum_{k=1}^N
%  \apply{z_{0,k}\otimes\delta_{x_{0,k}}}{
%  D\phi_{t}(y)^{-1}w(y)}_y
%  +
%  \sum_{k=1}^N
%  \sum_{j=1}^d
%  \apply{z_{0,k}^j\otimes\delta_{x_{0,k}}}{
%  \big(D^jD\phi_{t}(y)^{-1})w(y)
%  }_y
%  \\
%  &\quad
%  +
%  \sum_{k=1}^N
%  \sum_{j=1}^d
%  \apply{D\phi_{t}(x_{0,k})^{-1,T}z_{0,k}^j\otimes D^j\delta_{x_{0,k}}}{w}
%  \\
  &=
  \sum_{k=1}^N
  \apply{
  \big(
  D\phi_{t}(x_{0,k})^{-1,T}
  z_{0,k}
  +
  \sum_{j=1}^d
  \big(D^jD\phi_{t}(x_{0,k})^{-1}\big)^T
  z_{0,k}^j
  \big)\otimes\delta_{x_{0,k}}}{w}
  \\
  &\quad
  +
  \sum_{k=1}^N
  \sum_{j=1}^d
  \apply{D\phi_{t}(x_{0,k})^{-1,T}z_{0,k}^j\otimes D^j\delta_{x_{0,k}}}{w}
\end{align*}
showing that
$
  \mu_t
  =
  \sum_{k=1}^N
  \mu_{t,k}\otimes\delta_{x_{0,k}}
  +
  \sum_{k=1}^N
  \sum_{j=1}^d
  \mu_{t,k}^j\otimes D^j\delta_{x_{0,k}}
$
with
\begin{equation}
\begin{split}
  &\mu_{t,k}
  =
  D\phi_{t}(x_{0,k})^{-1,T}
  z_{0,k}
  +
  \sum_{j=1}^d
  \big(D^jD\phi_{t}(x_{0,k})^{-1}\big)^T
  z_{0,k}^j
  \\
  &\mu_{t,k}^j
  =
  D\phi_{t}(x_{0,k})^{-1,T}z_{0,k}^j
  \ .
  \label{eq:mu}
\end{split}
\end{equation}
The momentum $\rho_t$ can the be recovered as
\begin{align*}
  \apply{\rho_t}{w}
  &
  =
  \apply{\mu_t}{w\circ\phi_{t}}
  =
  \big(
  \sum_{k=1}^N
  \mu_{t,k}\otimes\delta_{x_{0,k}}
  +
  \sum_{k=1}^N
  \sum_{j=1}^d
  \mu_{t,k}^j\otimes D^j\delta_{x_{0,k}}
  \big)
  w\circ\phi_{t}
  \\
  &=
  \sum_{k=1}^N
  \mu_{t,k}\otimes\delta_{x_{t,k}}
  w
  +
  \sum_{k=1}^N
  \sum_{j=1}^d
  \mu_{t,k}^{j,T}
  Dw(D^j\phi_{t})(x_{0,k})
  \\
  &=
  \sum_{k=1}^N
  \mu_{t,k}\otimes\delta_{x_{t,k}}
  w
  +
  \sum_{k=1}^N
  \sum_{j=1}^d
  \big(
  \sum_{i=1}^d
  (D^i\phi_{t})(x_{0,k})^j
  \mu_{t,k}^i
  \big)
  \otimes
  D^j\delta_{x_{t,k}}w
\end{align*}
and hence the coefficients of the momentum $z_{t,k}$ and $z_{t,k}^j$ (confer
\eqref{eq:rho0}) are given by
$
  z_{t,k}
  =
  \mu_{t,k}
  $
  and
  $z_{t,k}^j
  =
  \sum_{i=1}^d
  (D^i\phi_{t})(x_{0,k})^j
  \mu_{t,k}^i
$.
We note that both $z_{t,k}^j$ and $\mu_{t,k}^j$ are coordinate vectors of
the first-order parts of the momentum in ordinary and Hamiltonian form
respectively. For each point $k$ and time $t$, these coordinate vectors thus 
represent two $d\times d$ tensors.

\subsection{Time Evolution}
Even though $\mu_{t,k}$ in \eqref{eq:mu} depend on the second order derivative of $\phi$, we will show
that the complete evolution in the zeroth and first-order case can be determined
by solving for the points $\phi_{t}(x_{k,0})$, the matrices $D\phi_{t}(x_{k,0})$, and
the vectors $\mu_{t,k}$. This will provide the computational representation we will use when
implementing the systems. 

Using \eqref{sys:hamilton}, $\phi_{t}$ evolves according to
\begin{align*}
  \partial_t\phi_{t}(y)
  &
  =
  \sum_{i=1}^d
  \int_\Omega
  \sum_{k=1}^N
  \big(
  \mu_{t,k}^T\otimes\delta_{x_{0,k}}
  +
  \sum_{j=1}^d
  \mu_{t,k}^j\otimes D^j\delta_{x_{0,k}}
  \big)
  K^i(\phi_{t}(x),\phi_{t}(y))
  dx\,
  e_i
  \\
  &
  =
  \sum_{i=1}^d
  \sum_{k=1}^N
  \big(
  \mu_{t,k}^T
  K^i(\phi_{t}(x_{0,k}),\phi_{t}(y))
  +
  \sum_{j=1}^d
  \mu_{t,k}^{j,T} 
  D_1K^i(\phi_{t}(x_{0,k}),\phi_{t}(y))
  D^j\phi_{t}(x_{0,k})
  \big)
  e_i
  \ .
\end{align*}
With scalar kernels, the trajectories $x_{t,k}$ are given by
\begin{align*}
  \partial_t\phi_{t}(x_{0,k})
  &
  =
  \sum_{l=1}^N
  \big(
  K(\phi_{t}(x_{0,l}),\phi_{t}(x_{0,k}))
  \mu_{t,l}
  +
  \sum_{j=1}^d
  \nabla_1K(\phi_{t}(x_{0,l}),\phi_{t}(x_{0,k}))^T
  D^j\phi_{t}(x_{0,l})
  \mu_{t,l}^j
  \big)
  \ .
\end{align*}

It is shown in \cite{younes_shapes_2010} that the evolution of
the matrix $D\phi_{t}(x_{k,0})$ is governed by
\begin{align*}
  &\partial_tD\phi_{t}(y)a
  =
  \sum_{i=1}^d\apply{\mu_t}{D_2K^i(\phi_{t}(x),\phi_{t}(y))D\phi_{t}(y)a}_xe_i
  \ .
\end{align*}
Inserting the Hamiltonian form of the higher-order momentum, each component
$(r,c)$ (\emph{r}ow/\emph{c}olumn) of the
matrix $D\phi_{t}(y)$ thus evolves according to
\begin{align*}
  &\partial_tD\phi_{t}(y)_r^c
  =
  \apply{\mu_t}{D_2K^r(\phi_{t}(x),\phi_{t}(y))D\phi_{t}(y)e_c}_x
  \\
  &\qquad
  =
  \int_\Omega
  \sum_{k=1}^N
  \big(
  \mu_{t,k}\otimes\delta_{x_{0,k}}
  +
  \sum_{j=1}^d
  \mu_{t,k}^j\otimes D^j\delta_{x_{0,k}}
  \big)
  D_2K^r(\phi_{t}(x),\phi_{t}(y))D\phi_{t}(y)e_c
  dx
  \\
  &\qquad
  =
  \sum_{k=1}^N
  \mu_{t,k}^T
  D_2K^r(\phi_{t}(x_{0,k}),\phi_{t}(y))D\phi_{t}(y)e_c
  \\
  &\qquad\quad
  +
  \sum_{k=1}^N
  \sum_{j=1}^d
  \mu_{t,k}^{j,T}
  \big(\sum_{i=1}^d\big(D_1^iD_2K^r(\phi_{t}(x_{0,k}),\phi_{t}(y))\big)\big(D^j\phi_{t}(x_{0,k})\big)^i\big)D\phi_{t}(y)e_c
  \ .
\end{align*}
With scalar kernels, the evolution at the trajectories is then
\begin{align*}
  \partial_tD\phi_{t}(x_{0,k})^c
  &=
  \sum_{l=1}^N
  \Big(
  \nabla_2K(\phi_{t}(x_{0,l}),\phi_{t}(x_{0,k}))^TD^c\phi_{t}(x_{0,k})
  \mu_{t,l}
  \\
  &\quad
  +
  \sum_{j=1}^d
  \big(D_1\nabla_2K(\phi_{t}(x_{0,l}),\phi_{t}(x_{0,k}))D^j\phi_{t}(x_{0,l})\big)^TD^c\phi_{t}(x_{0,k})
  \mu_{t,l}^j
  \Big)
  \ .
\end{align*}

The complete derivation of the evolution of $\mu_t$ is notationally heavy and can be found in
the supplementary material for the paper. Combining the evolution of $\mu_t$ with the expressions 
above, we arrive at the following result:
\begin{proposition}
The EPDiff equations in the scalar case with zeroth and first-order momenta are given in Hamiltonian
form by the system
\begin{equation}
\begin{split}
  &\partial_t\phi_{t}(x_{0,k})
  =
  \sum_{l=1}^N
  \big(
  K(x_{t,l},x_{t,k})
  \mu_{t,l}
  +
  \sum_{j=1}^d
  \nabla_1K(x_{t,l},x_{t,k})^T
  D^j\phi_{t}(x_{0,l})
  \mu_{t,l}^j
  \big)
  \\
  &\partial_tD\phi_{t}(x_{0,k})^c
  =
  \sum_{l=1}^N
  \Big(
  \nabla_2K(x_{t,l},x_{t,k})^TD^c\phi_{t}(x_{0,k})
  \mu_{t,l}
  \\
  &\qquad\qquad\qquad\quad
  +
  \sum_{j=1}^d
  \big(D_1\nabla_2K(x_{t,l},x_{t,k})D^j\phi_{t}(x_{0,l})\big)^TD^c\phi_{t}(x_{0,k})
  \mu_{t,l}^j
  \Big)
  \\
  &\partial_t\mu_{t,k}
  =
  -
  \sum_{l=1}^N
  \Big(
  \big(
  \mu_{t,k}^T
  \mu_{t,l}
  \big)
  \nabla_2K(x_{t,l},x_{t,k})
  \\
  &\qquad\quad\ \ 
  +\sum_{j=1}^d
  \big(
  \mu_{t,k}^{j,T}
  \mu_{t,l}
  \big)
  D_2\nabla_2K(x_{t,l},x_{t,k})
  D^j\phi_{t}(x_{0,k})
  \\
  &\qquad\quad\ \
  +\sum_{j=1}^d
  \big(
  \mu_{t,k}^T
  \mu_{t,l}^j
  \big)
  D_1\nabla_2K(x_{t,l},x_{t,k})
  D^j\phi_{t}(x_{0,l})
  \\
  &\qquad\quad\ 
  +\sum_{j,j'=1}^d
  \big(
  \mu_{t,k}^{j',T}
  \mu_{t,l}^j
  \big)
  D_2
  \big(
  D_1\nabla_2K(x_{t,l},x_{t,k})
  D^j\phi_{t}(x_{0,l})
  \big)
  D^{j'}\phi_{t}(x_{0,k})
  \Big)
  \\
  &\mu_{t,k}^j
  =
  D\phi_{t}(x_{0,k})^{-1,T}z_{0,k}^j
  \ .
  \label{sys:epdiff-hamilton}
\end{split}
\end{equation}
  
\end{proposition}
Note that both $x_{1,k}=\phi_{01}^v(x_{0,k})$ and $D\phi_{01}^v(x_{0,k})$ are
provided by the system and hence can be used to evaluate a similarity measure
that incorporates first-order information. As in the zeroth order
case, the entire evolution can be recovered by the initial conditions for the
momentum.

\section{Locally Affine Transformations}
\label{sec:loc-affine}
The Polyaffine and Log-Euclidean Polyaffine \cite{arsigny_polyrigid_2005,arsigny_fast_2009} 
frameworks model locally affine transformations 
using matrix logarithms.
The higher-order momenta and partial derivatives of kernels can be seen as the LDDMM sibling of the
Polyaffine methods, and
diffeomorphism paths generated by higher-order momenta, in
particular, momenta of zeroth and first-order, can locally approximate all affine
transformations with linear component having positive determinant. The approximation 
will depend only on how fast the kernel
approaches zero towards infinity. The manifold structure of $G_V$ provides this
result immediately. 
Indeed, let $\phi(x)=Ax+b$ be an affine transformation with $\det(A)>0$.  We
define a path $\phi_t$ of finite energy such that $\phi_1\approx \phi$ which
shows that $\phi_1\in G_V$ and can be reached in the framework.
The matrices of positive determinant is path connected so we can let $\psi_t$ be
a path from $\Id_d$ to $A$ and define
$\tilde{\psi}_t(x)=\psi_tx+bt$. Then with
$\tilde{v}_t(x)=(\partial_t\psi_t)\tilde{\psi}_t^{-1}(x)+b$, we have
$\partial_t\tilde{\psi}_t(x)=(\partial_t\psi_t)x+b=\tilde{v}_t\circ\tilde{\psi}_t(x)$
and 
\begin{equation*}
  x+\int_0^1\tilde{v}_t\circ\tilde{\psi}_t(x)dt
  =
  x+\int_0^1(\partial_t\psi_t)x+bdt
  =
  \phi(x)
  \ .
\end{equation*}
Now use that
$(\partial_t\psi_t)\tilde{\psi}_t^{-1}(x)=(\partial_t\psi_t)(\psi_t)^{-1}(x-bt)$ and let
the $M_t=(m_{1,t}\ldots m_{d,t})$ be the $t$-dependent matrix $(\partial_t\psi_t)(\psi_t)^{-1}$ so that
the first term of $\tilde{v}_t(x)$ equals $M_t(x-bt)$.
Then choose a radial kernel, e.g. a Gaussian $K_\sigma$, and define the
approximation $v_t$ of $\tilde{v}_t$ by
\begin{equation}
  v_t(x)
  =
  \sum_{j=1}^dD^j_{\tilde{\psi}_t(0)}K_\sigma(x)m_{j,t} +K_\sigma(\tilde{\psi}_t(0),x)b
  \ .
  \label{eq:vel-affine}
\end{equation}
The path $\phi_{01}^v$ generated by $v_t$ then has finite energy, and
\begin{equation*}
  \phi_{01}^v(x)
  =
  x+\int_0^1v_t\circ\phi_{0t}^v(x)dt
  \approx
  \phi(x)
\end{equation*}
with the approximation depending only on the kernel scale $\sigma$.
Note that the affine transformations with linear components having negative
determinant can in a similar way be reached by starting the integration at a
diffeomorphism with negative Jacobian determinant.

\newcommand{\transl}{\mathrm{Tsl}}
\newcommand{\lin}{\mathrm{Lin}}
In the experiments section, we will illustrate the locally affine
transformations encoded by zeroth and first-order momenta, and, therefore, it will 
be useful to introduce a notation for these momenta. We encode the translational part
of either the momentum or velocity using the notation
\begin{equation*}
  \transl_{x}(b)=K_{\sigma}(x,\cdot)b
\end{equation*}
and the linear part by
\begin{equation*}
  \lin_{x}(M)
  =
  \sum_{j=1}^d
  D^j_{x}K_\sigma(\cdot)m_j
\end{equation*}
with $m_1,m_j$ being the columns of the matrix $M$. Equation
\eqref{eq:vel-affine} can then be written
\begin{equation}
  v_t(x)
  =
  \lin_{\tilde{\psi}_t(0)}(M_t) +\transl_{\tilde{\psi}_t(0)}(b)
  \ .
\end{equation}
We emphasize that though we mainly focus on zeroth and first-order momenta, the mathematical 
construction allows any order momenta permitted by the smoothness of 
the kernel at order zero.

\section{Variations of the Initial Conditions}
\label{sec:variations}
In Algorithm~\ref{alg:alg1}, we used the variation of the EPDiff equations when
varying the initial conditions and in particular the backwards gradient
transport. We discuss both issues here.

A variation $\delta\rho_0$ of the initial momentum will induce a variation of the
system \eqref{sys:epdiff-hamilton}. By differentiating that system, we get the
time evolution of the variation. To ease notation, we assume the kernel is
scalar on the form $K(x,y)=\gamma(|x-y|^2)$ and write
$\gamma_{t,lk}=K(x_{t,l},x_{t,k})$.\footnote{The subscript notation is used in
accordance with \cite{younes_shapes_2010}. Please note that 
$\gamma_{t,lk}$ contains \emph{three} separate indices, i.e. the time $t$ and
the point indices $l$ and $k$.} Variations of the kernel and kernel derivatives
such as the entity $\delta \nabla_1K(x_{t,l},x_{t,k})$ below depend only on the
variation of point trajectories $\delta x_{t,l}$ and $\delta x_{t,k}$. The full expressions for these
parts are provided in supplementary material for the paper. The variation of the
point trajectories in the derived
system then takes the form
\begin{equation*}
\begin{split}
  &\partial_t\delta\phi_{t}(x_{0,k})
  =
  \sum_{l=1}^N
  \big(
  \delta K(x_{t,l},x_{t,k})
  \mu_{t,l}
  +
  \gamma_{t,lk}
  \delta
  \mu_{t,l}
  \big)
  \\
  &\qquad
  +
  \sum_{l=1}^N
  \sum_{j=1}^d
  \big(
  \delta \nabla_1K(x_{t,l},x_{t,k})^T
  D^j\phi_{t}(x_{0,l})
  \mu_{t,l}^j
  +
  \nabla_1K(x_{t,l},x_{t,k})^T
  \delta D^j\phi_{t}(x_{0,l})
  \mu_{t,l}^j
  \\
  &\qquad\qquad
  +
  \nabla_1K(x_{t,l},x_{t,k})^T
  D^j\phi_{t}(x_{0,l})
  \delta
  \mu_{t,l}^j
  \big)
\end{split}
\end{equation*}
The similar expressions for the evolution of $\delta\mu_{t,k}$ and $\delta
D\phi_{t}(x_{0,k})$ are provided in the supplementary material.
The variation of $\mu_{t,k}^j$ is available as
\begin{align*}
  &\delta\mu_{t,k}^j
  =
  -\big(
  D\phi_{t}(x_{0,k})^{-1}
  \delta D\phi_{t}(x_{0,k})
  D\phi_{t}(x_{0,k})^{-1}
  \big)^Tz_{0,k}^j
  +
  D\phi_{t}(x_{0,k})^{-1,T}\delta z_{0,k}^j
  \ .
\end{align*}
However, when computing the backwards transport, we will need to remove the
dependency on $\delta z_{0,k}^j$ which is only available for forward
integration. Instead, by writing the evolution of $\mu_{t,k}^j$ in the form
\begin{align*}
  &\partial_t \mu_{t,k}^j
  =
  \partial_t D\phi_{t}(x_{0,k})^{-1,T}z_{0,k}^j
  =
  -\big(
  D\phi_{t}(x_{0,k})^{-1}
  \partial_t D\phi_{t}(x_{0,k})
  D\phi_{t}(x_{0,k})^{-1}
  \big)^Tz_{0,k}^j
  \\
  &\qquad
  =
  -D\phi_{t}(x_{0,k})^{-1,T}
  \partial_t D\phi_{t}(x_{0,k})^T
  \mu_{t,k}^j
  \ ,
\end{align*}
we get the variation
\begin{align*}
  &\partial_t \delta\mu_{t,k}^j
  =
  -\delta D\phi_{t}(x_{0,k})^{-1,T}
  \partial_t D\phi_{t}(x_{0,k})^T
  \mu_{t,k}^j
  -
  D\phi_{t}(x_{0,k})^{-1,T}
  \partial_t \delta D\phi_{t}(x_{0,k})^T
  \mu_{t,k}^j
  \\
  &\qquad\qquad
  -
  D\phi_{t}(x_{0,k})^{-1,T}
  \partial_t D\phi_{t}(x_{0,k})^T
  \delta \mu_{t,k}^j
  \ .
\end{align*}

\subsection{Backwards Transport}
The correspondence between initial momentum $\rho_0$ and end diffeomorphism
$\phi_{01}^v$ asserted by the EPDiff equations allows us to view the similarity 
measure $U(\phi_{01}^v)$ as a function of $\rho_0$.
Let $A$ denote the result of integrating the system for the variation of the
initial conditions from $t=0$ to $t=1$ such that $w=A\delta\rho_0\in V$ for
a variation $\delta\rho_0$. 
We then get a corresponding variation $\delta U$ in the similarity measure.
To compute the gradient of $U$ as a function of $\rho_0$, we have
\begin{align*}
  \delta U(\phi_{01}^v)
  =
  \ip{\nabla_{\phi_{01}^v}U,w}_V
  =
  \ip{\nabla_{\phi_{01}^v}U,A\delta\rho_0}_V
  =
  \ip{A^T\nabla_{\phi_{01}^v}U,\delta\rho_0}_{V^*}
  \ .
\end{align*}
Thus, the $V^*$-gradient of $\nabla_{\rho_0}U$ is given by $A^T\nabla_{\phi_{01}^v}U$.
The gradient can equivalently be computed in momentum space at both endpoints
of the diffeomorphism path using the map $P$ defined in
Proposition~\ref{prop:grad}.

The complete system for the variation of the initial conditions is a linear ODE,
and, therefore, there exists a time-dependent matrix $M_t$ such that the ODE
\begin{align*}
  \partial_t y_t
  =
  M_t
  y_t
\end{align*}
has the variation as a solution $y_t$. It is shown in
\cite{younes_shapes_2010} that, in such cases, solving the backwards transpose system
\begin{equation}
  \partial_t w_t
  =
  -M_t^T
  w_t
  \label{sys:backwards}
\end{equation}
from $t=1$ to $t=0$ provides the value of $A^Tw$. Therefore, we can obtain 
$\nabla_{\rho_0}U$ by solving the transpose
system backwards. The components of $M_t$ can be identified by writing the evolution 
equations for the variation in matrix form. This provides $M_t^T$ and allows the 
backwards integration of the system \ref{sys:backwards}. The components of the 
transpose matrix $M_t$ are provided in the supplementary material for the paper.

\section{Experiments}
\label{sec:experiments}
In order to demonstrate the efficiency, compactness, and interpretability of representations using
higher-order momenta, we perform four sets of experiments. First, we provide 
four examples illustrating the type of deformations produced by zeroth and first-order 
momenta and the relation to the Polyaffine framework.
We then use point based matching using first-order information to show how
complicated warps that would require many parameters with zeroth order deformation
atoms can be generated with very compact representations using
higher-order momenta.
We underline the point that higher-order momenta allow low-dimensional transformations to
be registered using correspondingly low-dimensional representations:
we show how synthetic test images generated by a low-dimensional
transformation can be registered using only one deformation atom 
when representing using first-order momenta and using the first-order similarity 
measure approximation \eqref{eq:tildeU1}.
We further emphasize this point by registering articulated movement using only one
deformation atom per rigid part, and thus exemplify a natural representation
that reduces the number of deformation atoms and the ambiguity in the placement of 
the atoms while also reducing the degrees of freedom in the representation.
Finally, we illustrate how higher-order momenta in a natural way
allow registration of human brains with progressing atrophy. We describe the
deformation field throughout the ventricles using few deformation atoms, and
we thereby suggest a method for detecting anatomical change using few
degrees of freedom. In addition, the volume expansion can be directly interpreted
from the parameters of the deformation atoms.
We start by briefly describing the similarity measures
used throughout the experiments.

For the point examples below, we register moving points $x_1,\ldots,x_N$ against 
fixed points $y_1,\ldots,y_N$. In addition, we match first-order information by
specifying values of $D^j_{x_k}\phi$. This is done compactly by providing
matrices $Y_k$ so that we seek $D_{x_k}\phi=Y_k$ for all $k=1,\ldots,N$.
The similarity measure is simple sum of squares, i.e.
\begin{equation*}
  U(\phi)
  =
  \sum_{i=1}^N
  \|\phi(x_k)-y_k\|^2
  +
  \|D_{x_k}\phi-Y_k\|^2
\end{equation*}
using the matrix $2$-norm. This amounts to fitting $\phi$ against a locally affine 
map with translational components $y_k$ and linear components $Y_k$.
For the image cases, we use $L^1$-similarity to build
the first-order approximation \eqref{eq:tildeU1} with the smoothing kernel
$K_s$ being Gaussian of the same scale as the LDDMM
kernel.

\subsection{First Order Illustrations}
To visually illustrate the deformation generated by higher-order momenta,
we show in Figure~\ref{fig:shots} the generated
deformations on an initially square grid with four different first-order initial momenta. 
The deformation locally model the linear part of affine transformations and the
the locality is determined by the Gaussian kernel that in the examples has 
scale $\sigma=8$ in grid units. Notice for the rotations
that the deformation stays diffeomorphic in the presence of conflicting forces. 
The similarity between the examples and the deformations generated in the Polyaffine
framework \cite{arsigny_fast_2009} underlines the viewpoint that the
registration using higher-order momenta constitutes the LDDMM sibling of the Polyaffine framework.
\begin{figure}[t]
\begin{center}
  \parbox{0.99\columnwidth}{
\begin{center}
  \subfigure[Expansion]{\includegraphics[width=0.46\columnwidth,trim=0 0 0 0,clip=true]{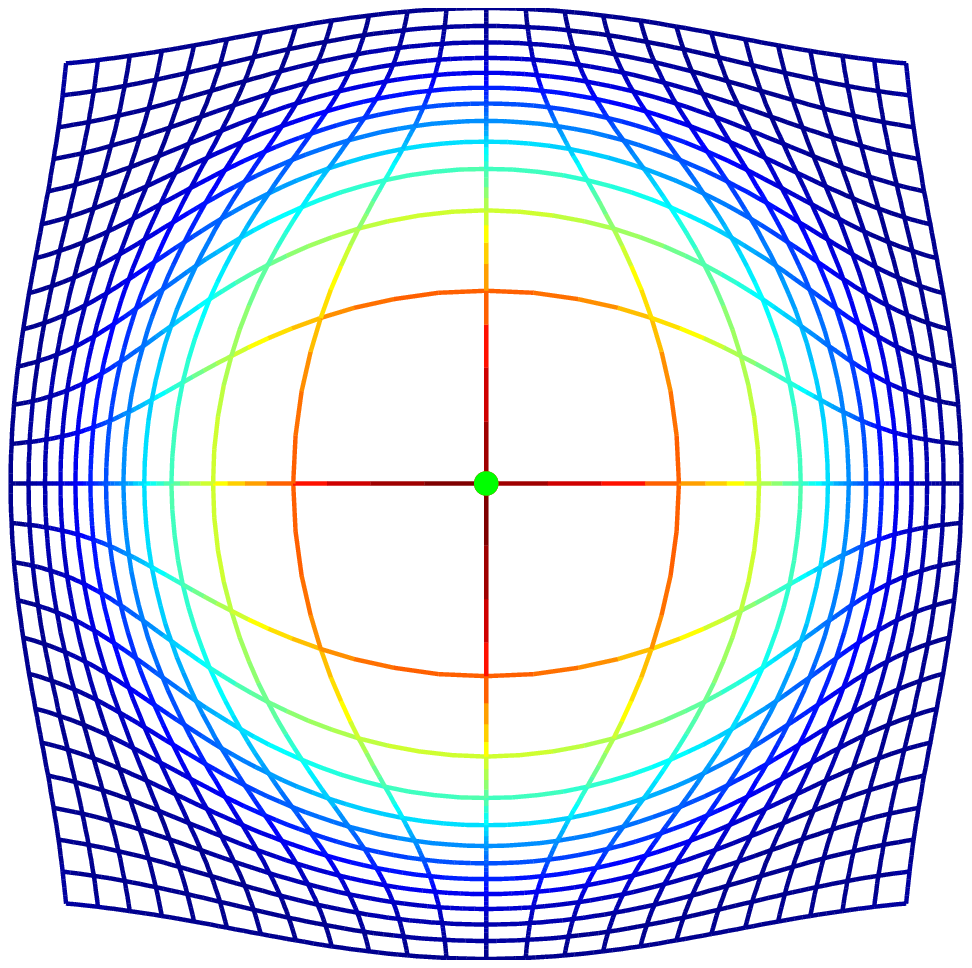}}
  \subfigure[Contraction]{\includegraphics[width=0.46\columnwidth,trim=0 0 0 0,clip=true]{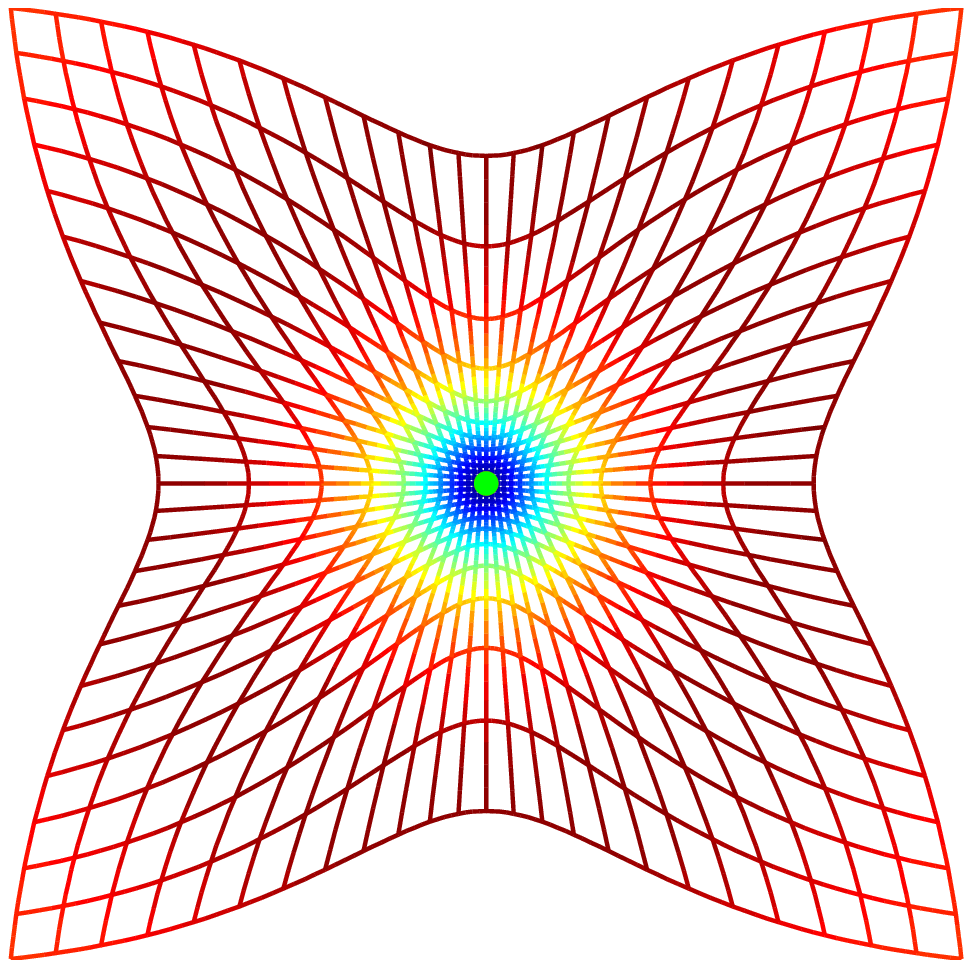}}
  \subfigure[Rotation ($-\pi/2$)]{\includegraphics[width=0.46\columnwidth,trim=0 0 0 0,clip=true]{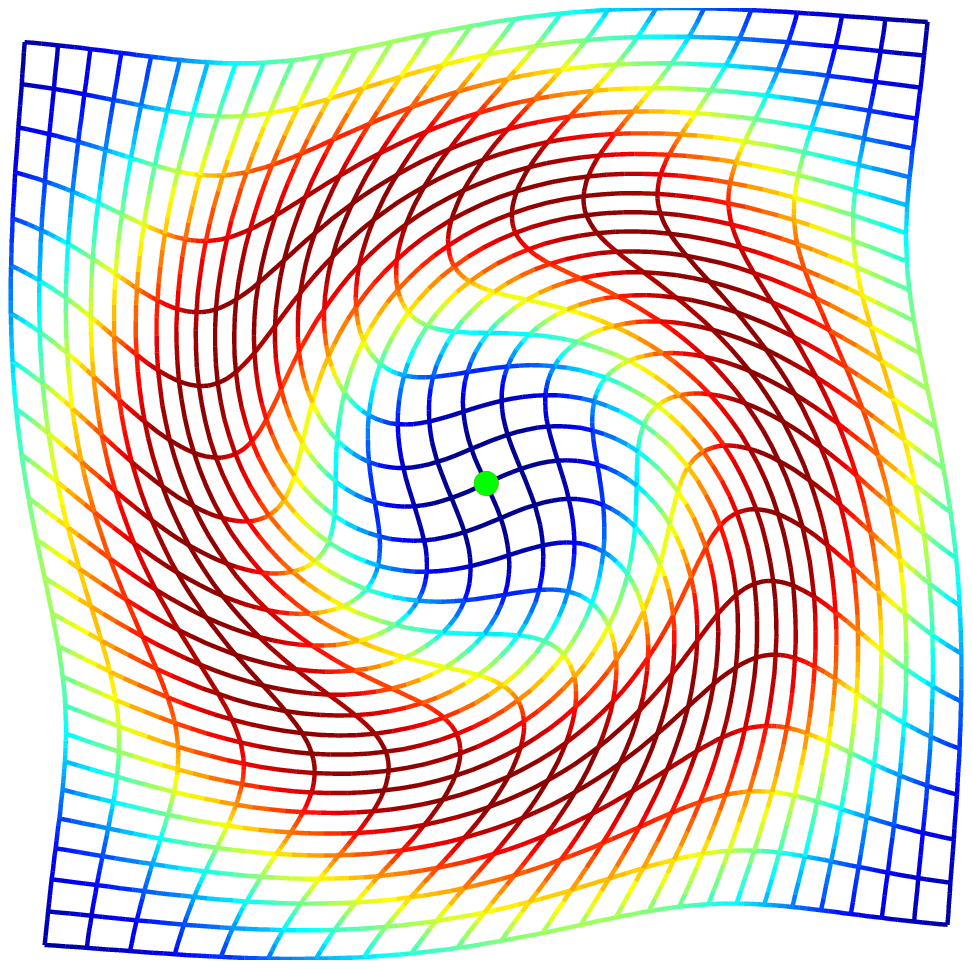}}
  \subfigure[Two rotations ($\pi/2$)]{\includegraphics[width=0.49\columnwidth,trim=50 60 40 70,clip=true]{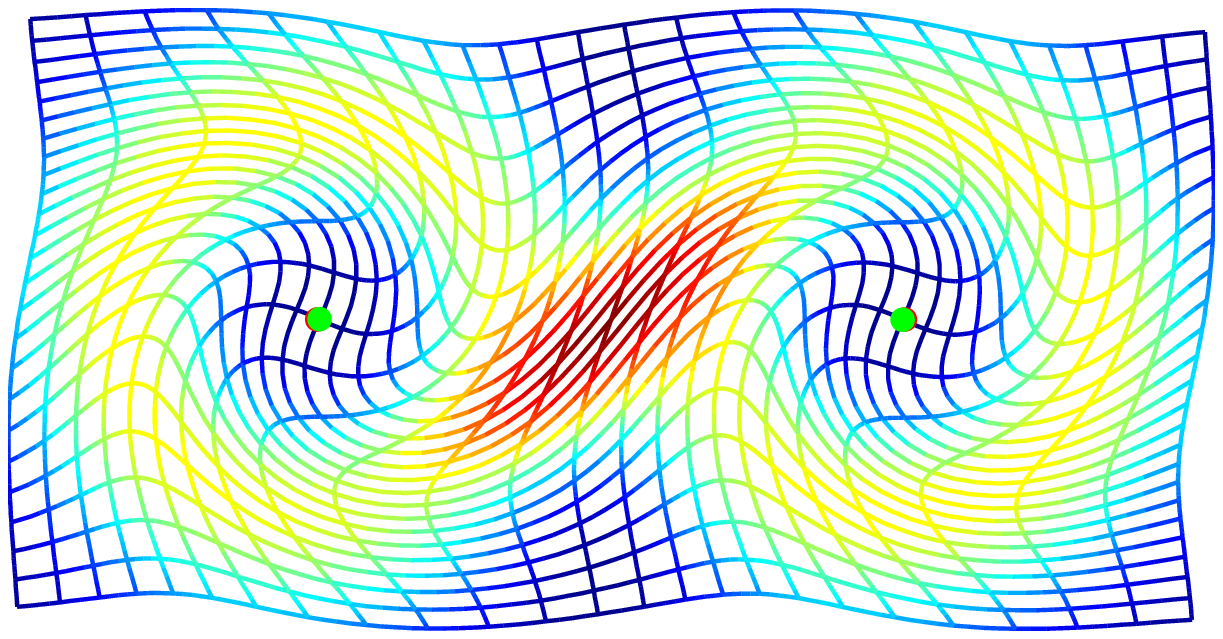}}
\end{center}
  }
\end{center}
\caption{
The effect of the generated
deformation on an initially square grid for several initial first-order momenta:
Using the notation of section~\ref{sec:loc-affine},
(a) expansion $\rho_0=\lin_{0}(\Id_2)$; (b) contraction
$\rho_0=\lin_{0}(-\Id_2)$; (c) rotation 
$\rho_0=\lin_{0}(\mathrm{Rot}(v))$,
$v=-\pi/2$; (d) two rotations $v=\pi/2$. The kernel is Gaussian with $\sigma=8$ in
grid units, and the grids are colored with
the trace of Cauchy-Green strain tensor (log-scale).
Notice the locality of the deformation caused by the finite scale of the
kernel, and that the deformation stays diffeomorphic even when two rotations
force conflicting movements.
}
\label{fig:shots}
\end{figure}

\subsection{First Order Point Registration}
Figure~\ref{fig:matches} presents simple point based matching results with
first-order information. The lower points (red) are matched against the upper points
(black) with match against expansion $D_\phi(x_k)=2\Id_2$ and rotation
$D_\phi(x_k)=\mathrm{Rot}(v)=\begin{pmatrix}\cos(v),\sin(v)\\-\sin(v),\cos(v)\end{pmatrix}$ for
$v=\mp \pi/2$. The optimal diffeomorphisms exhibit the expected expanding and
turning effect, respectively. We stress that the deformations are generated
using only two deformations atoms with combined 12 parameters.
Representing equivalent deformation using zeroth order momenta would require 
a significantly increased number of atoms and a correspond increase in the 
number of parameters.
\begin{figure}[t]
  \parbox{0.97\columnwidth}{
\begin{center}
  \subfigure[Match with dilations (expansion)]{\includegraphics[width=0.46\columnwidth,trim=40 0 40 0,clip=true]{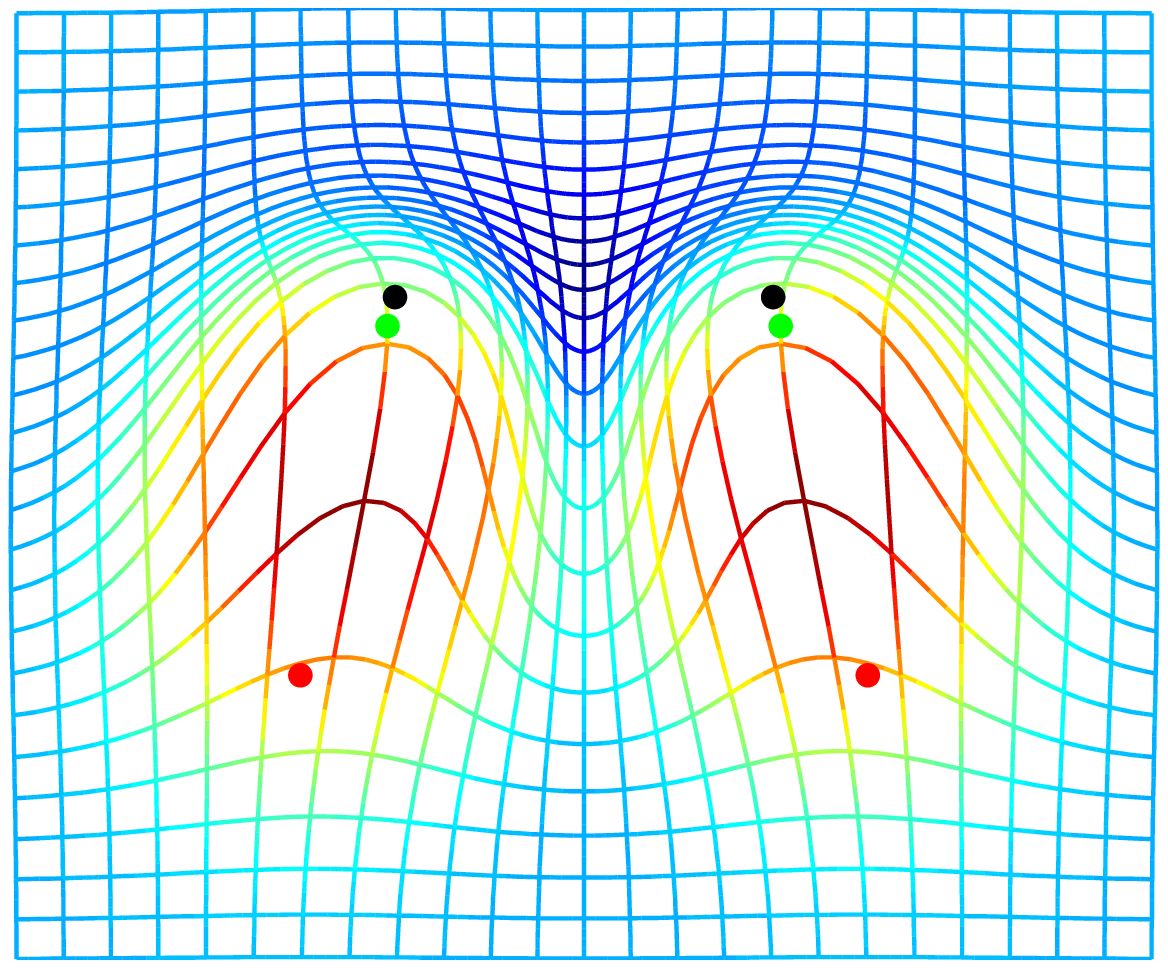}}
  \subfigure[Match with rotations ($-\pi/2$ and $\pi/2$)]{\includegraphics[width=0.46\columnwidth,trim=40 0 40 0,clip=true]{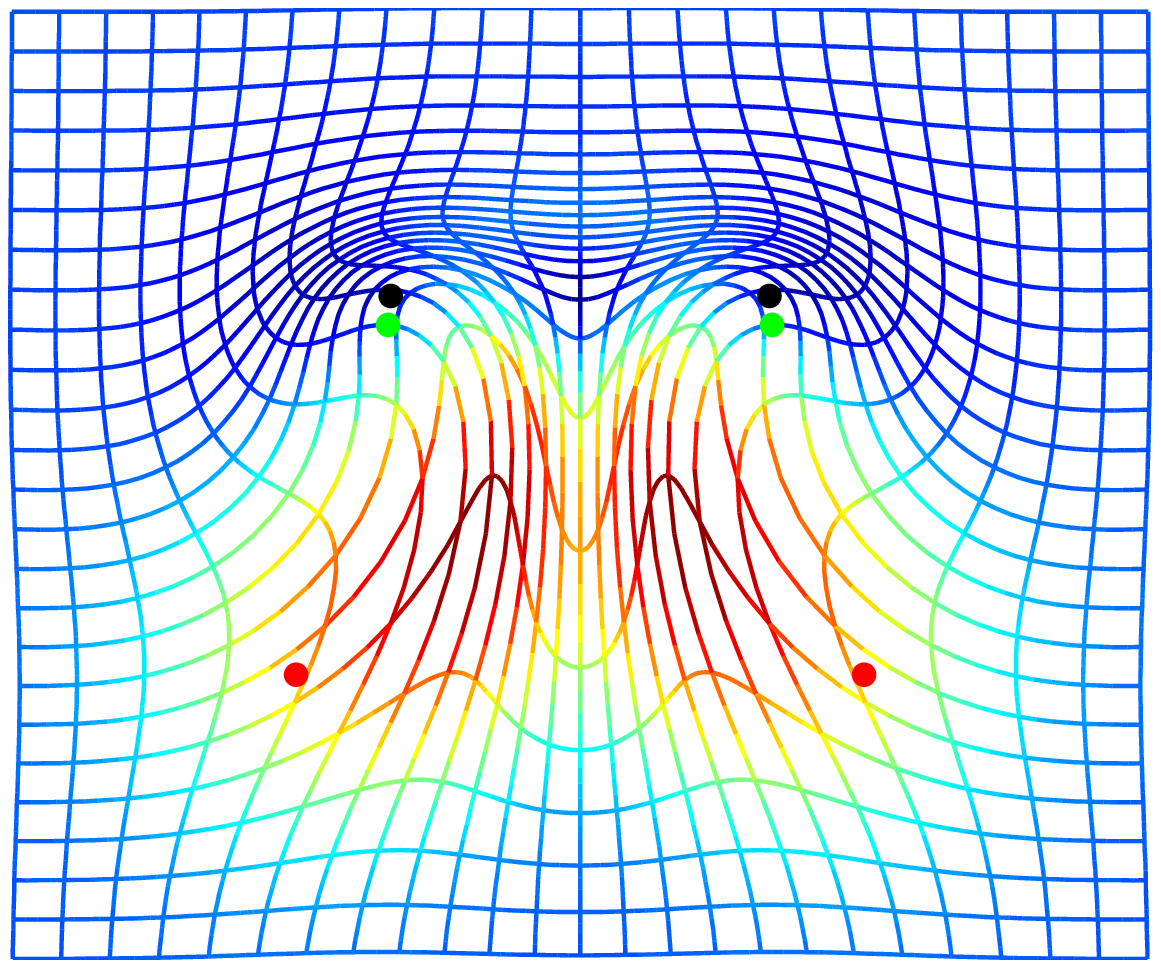}}
\end{center}
  }
\caption{
Two moving points (red) are matched against two fixed points (black) with
results (green) and with
match against (a) expansion $D_\phi(x_k)=2\Id_2$, $i=1,2$; and (b) rotation
$D_\phi(x_k)=\mathrm{Rot}(v)$, $v=\mp \pi/2$, $i=1,2$.
The kernel is Gaussian with $\sigma=8$ in
grid units, and the grids are colored with
the trace of Cauchy-Green strain tensor (log-scale).
}
\label{fig:matches}
\end{figure}

\subsection{Low Dimensional Image Registration}
We now exemplify how higher-order momenta allow low-dimensional transformations to
be registered using correspondingly low-dimensional representations.
We generate two test images by applying two linear transformations, an dilation
and a rotation, to a binary image of a square, confer the moving images 
(a) and (e) in Figure~\ref{fig:simple-images}.
By placing one deformation atom in the center of each fixed image and by
using the similarity measure approximation \eqref{eq:tildeU1}, we can
successfully register the moving and fixed images. The result and difference
plots are shown in Figure~\ref{fig:simple-images}. The dimensionality of the
linear transformations generating the moving images is equal to the number of
parameters for the deformation atom. A registration using zeroth order momenta
would need more than one deformation atom which would result in a number of
parameters larger than the dimensionality.
The scale of the Gaussian kernel used for the registration is 50 pixels.
\begin{figure}[t]
\begin{center}
  \parbox{0.99\columnwidth}{
  \subfigure[Moving image]{\includegraphics[width=0.24\columnwidth,trim=40 20 40 20,clip=true]{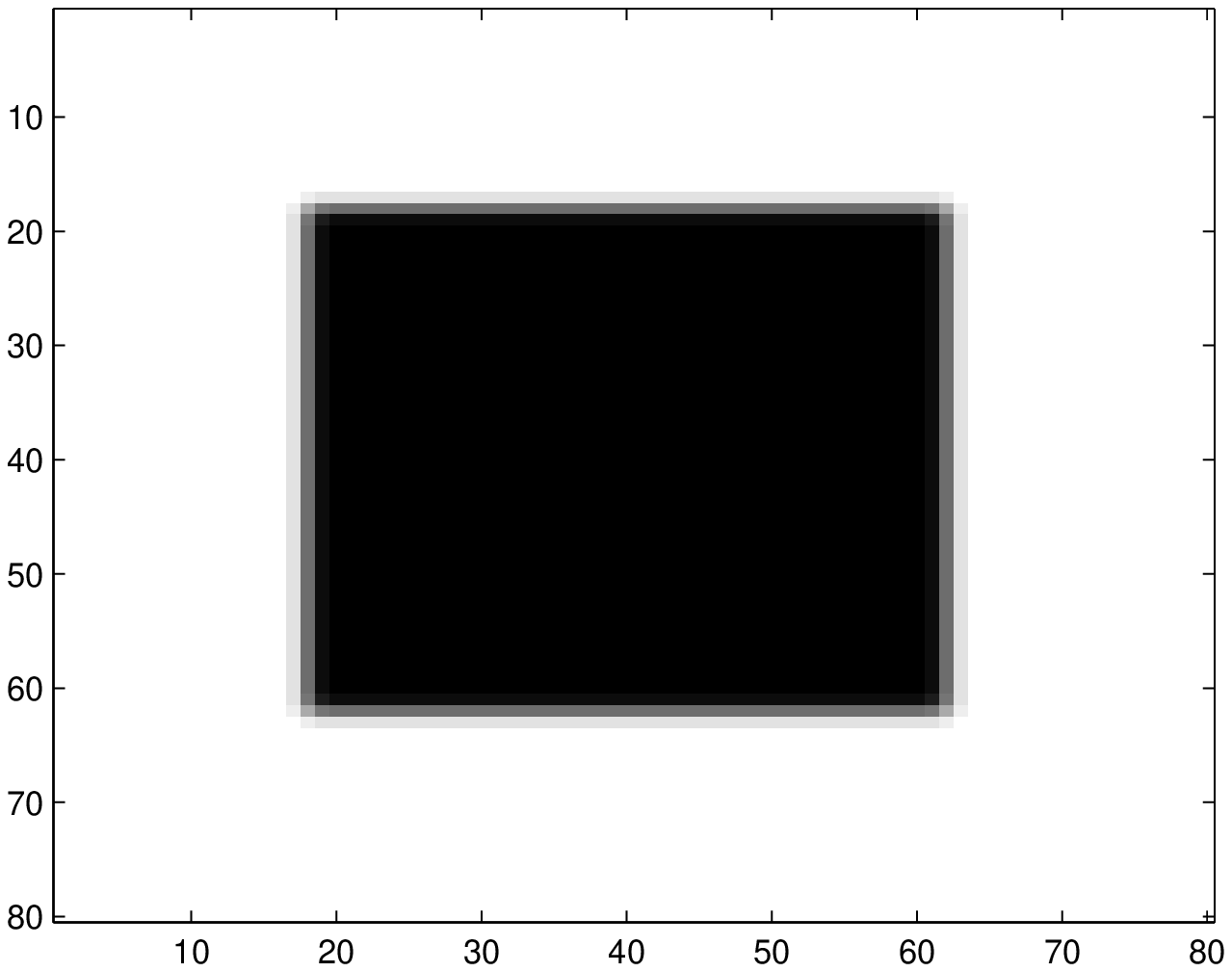}}
  \subfigure[Fixed image]{\includegraphics[width=0.24\columnwidth,trim=40 20 40 20,clip=true]{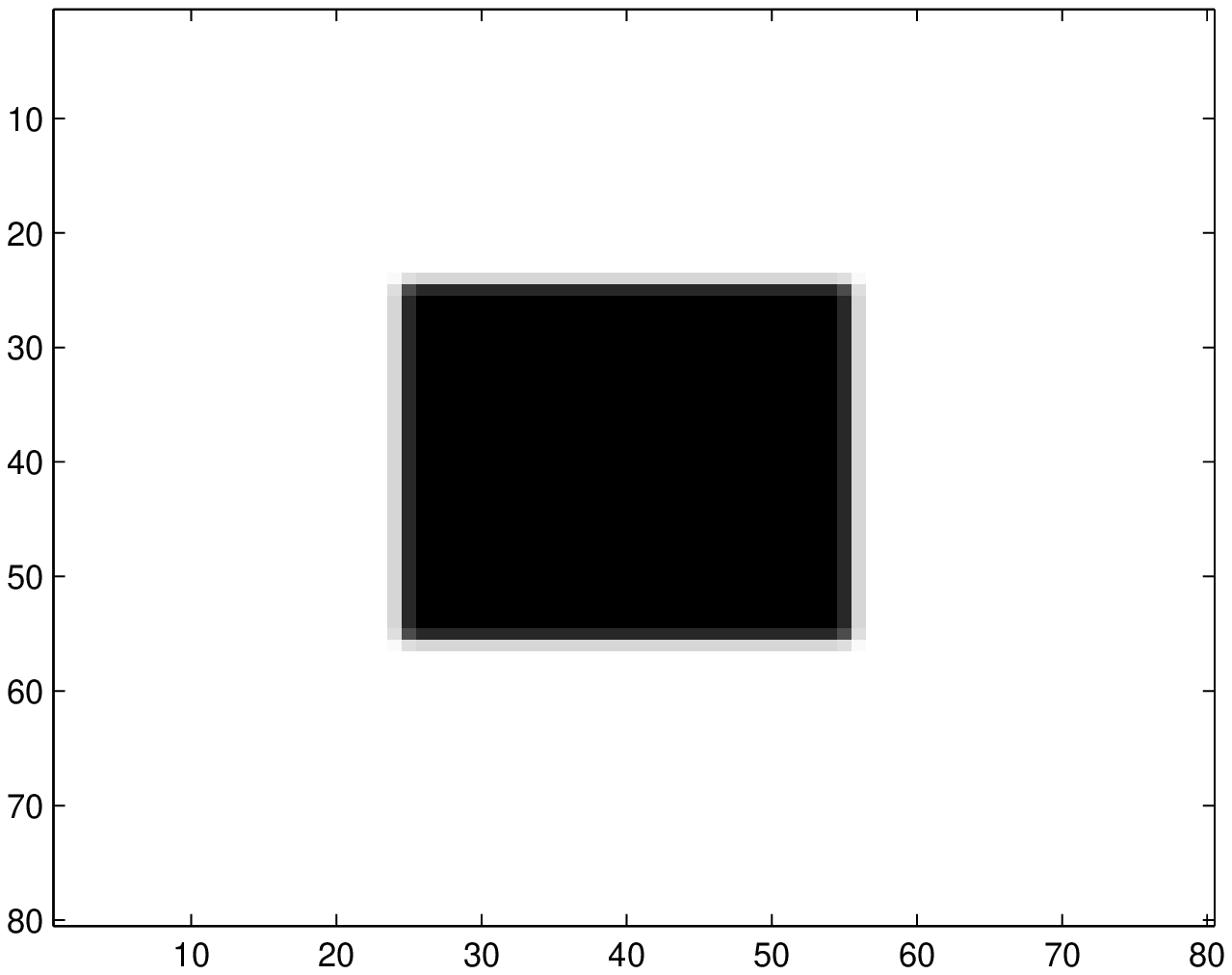}}
  \subfigure[Registration result]{\includegraphics[width=0.24\columnwidth,trim=40 20 40 20,clip=true]{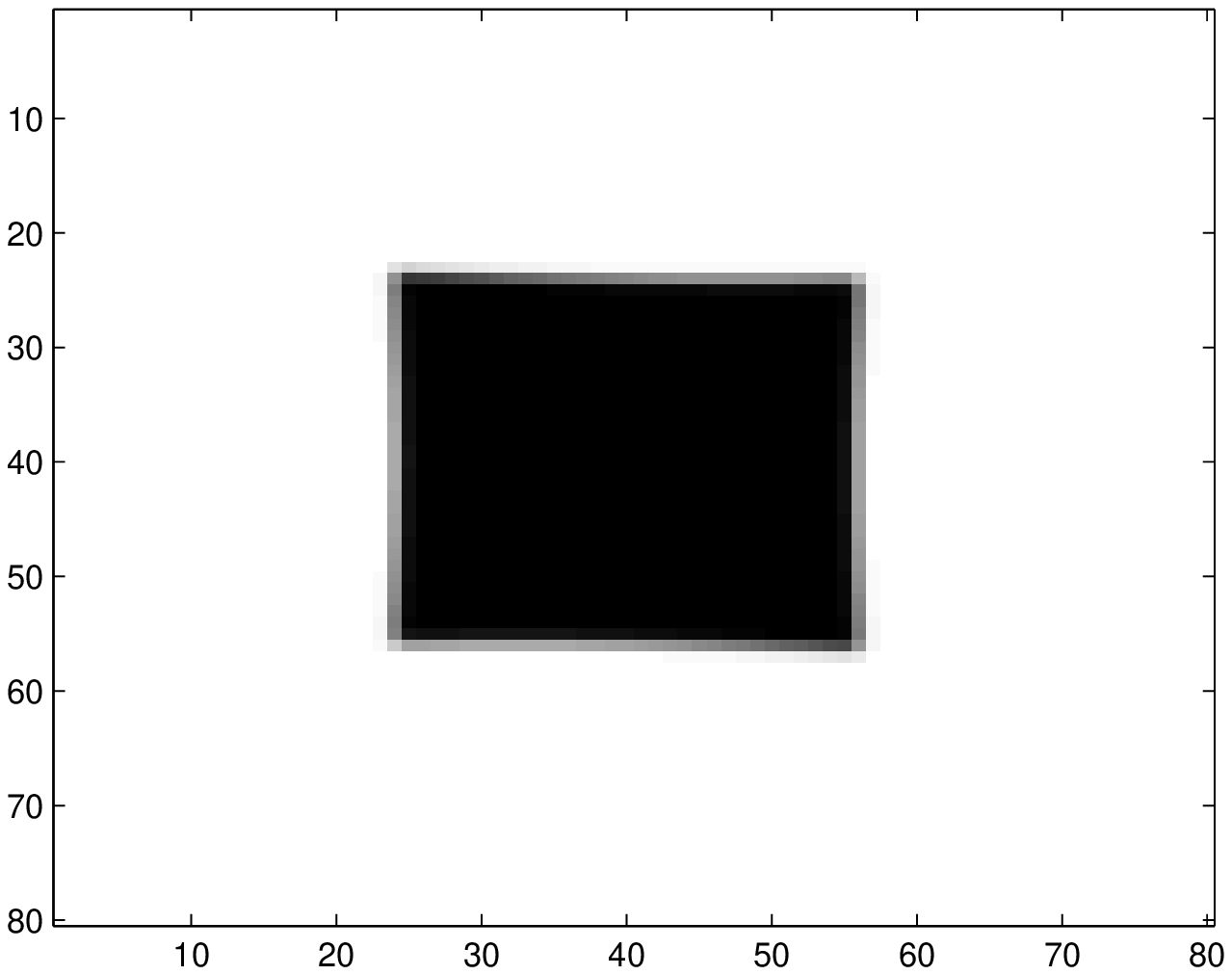}}
  \subfigure[Difference]{\includegraphics[width=0.24\columnwidth,trim=40 20 40 20,clip=true]{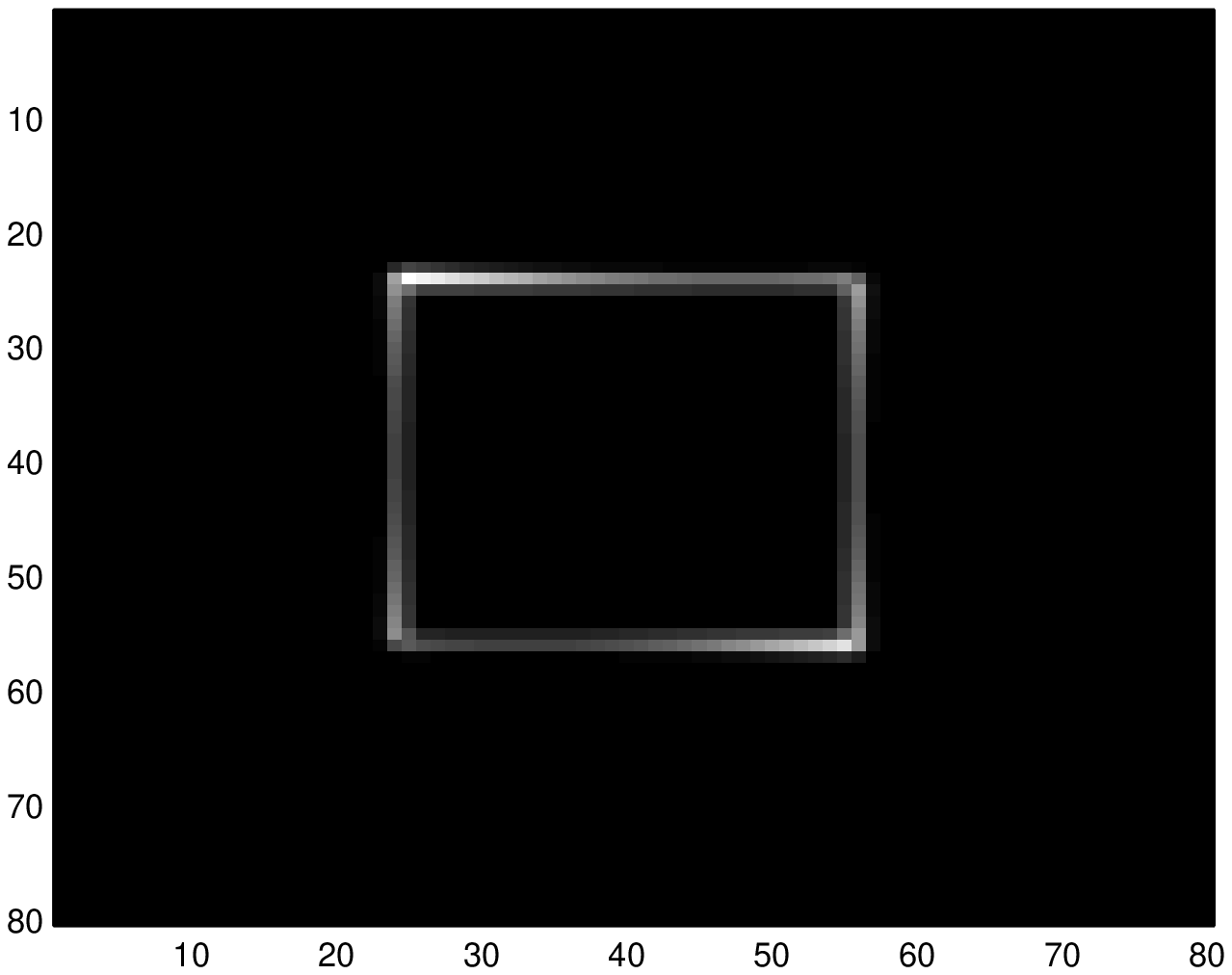}}
  \line(1,0){370}\linebreak
  \subfigure[Moving image]{\includegraphics[width=0.24\columnwidth,trim=40 20 40 20,clip=true]{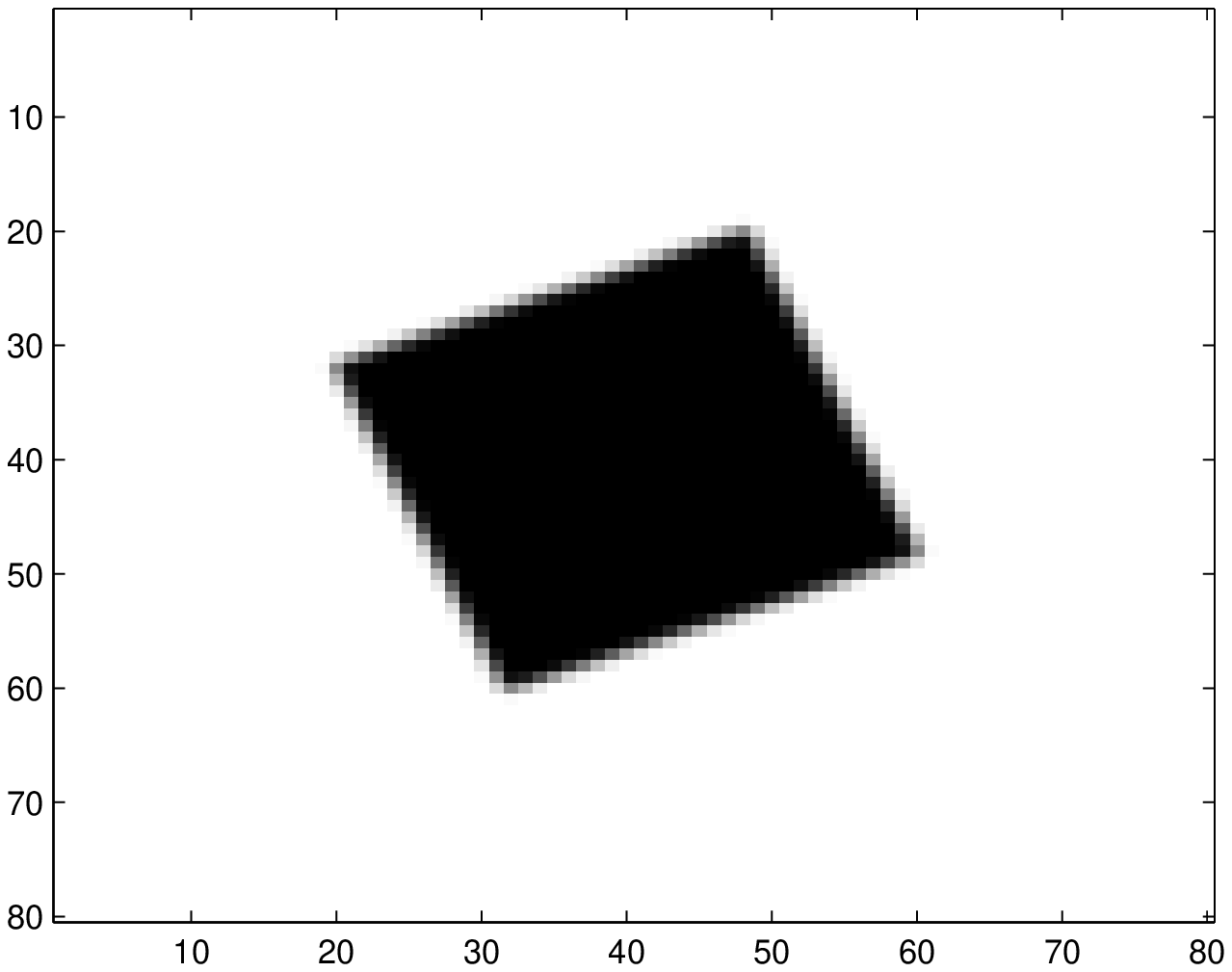}}
  \subfigure[Fixed image]{\includegraphics[width=0.24\columnwidth,trim=40 20 40 20,clip=true]{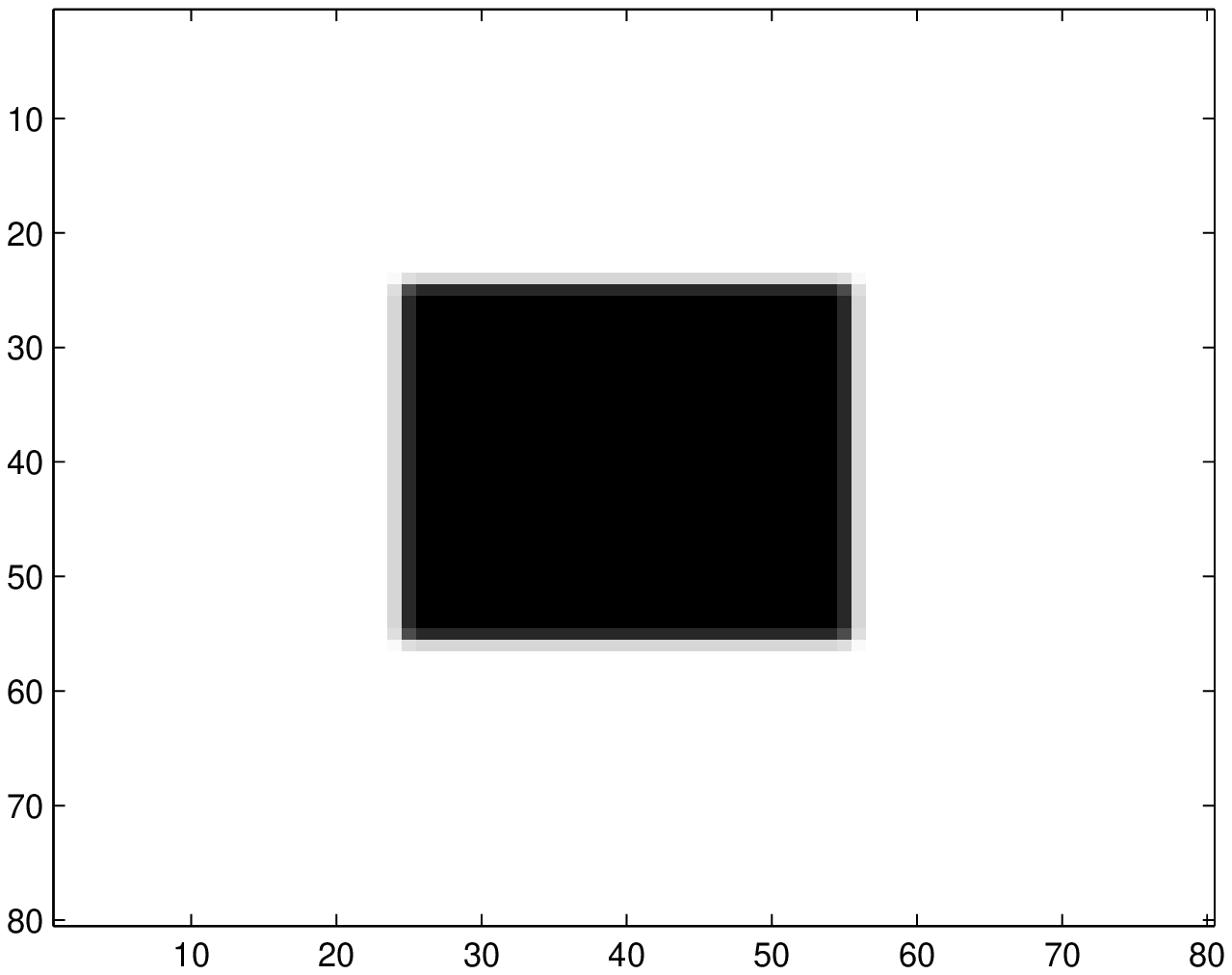}}
  \subfigure[Registration result]{\includegraphics[width=0.24\columnwidth,trim=40 20 40 20,clip=true]{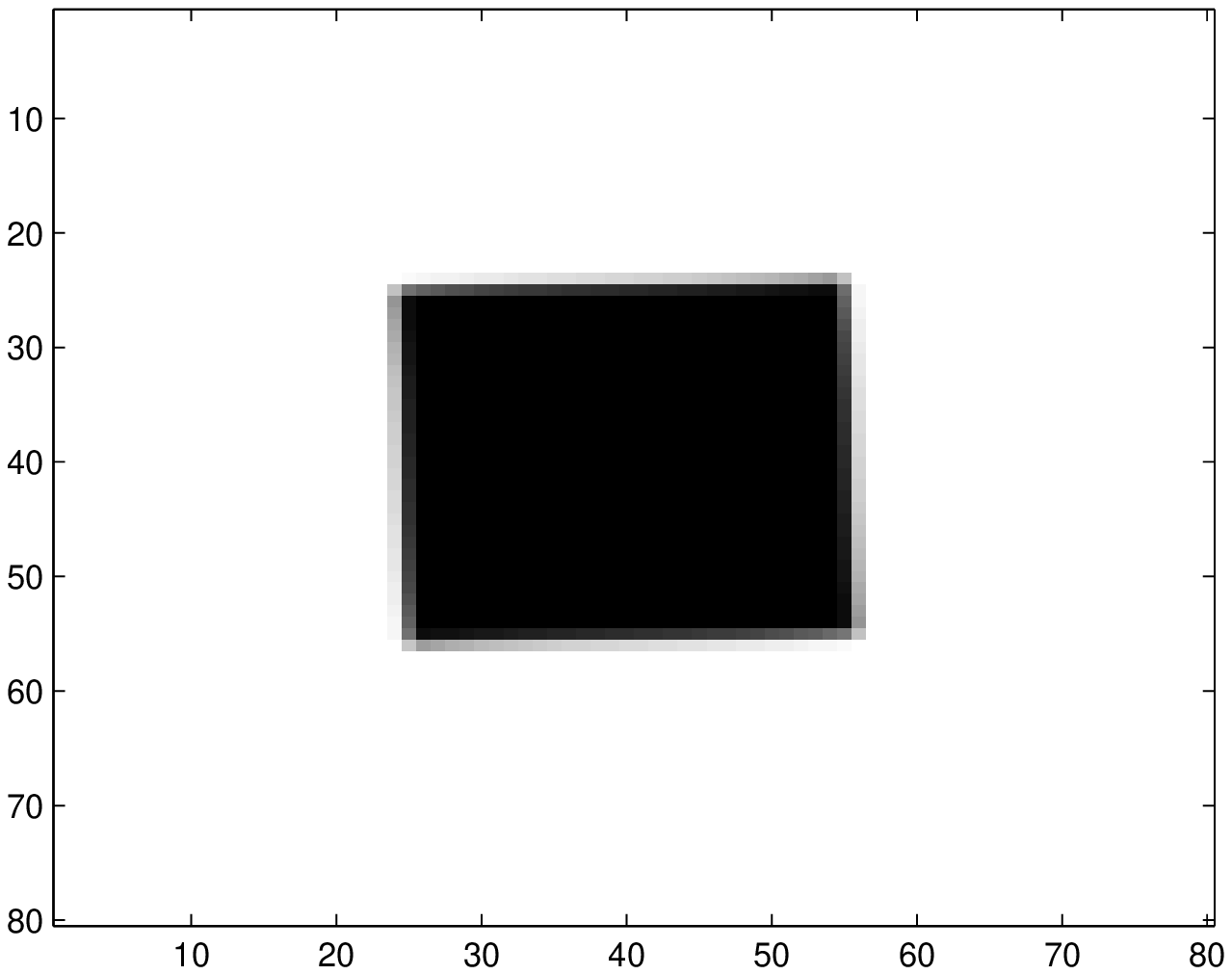}}
  \subfigure[Difference]{\includegraphics[width=0.24\columnwidth,trim=40 20 40 20,clip=true]{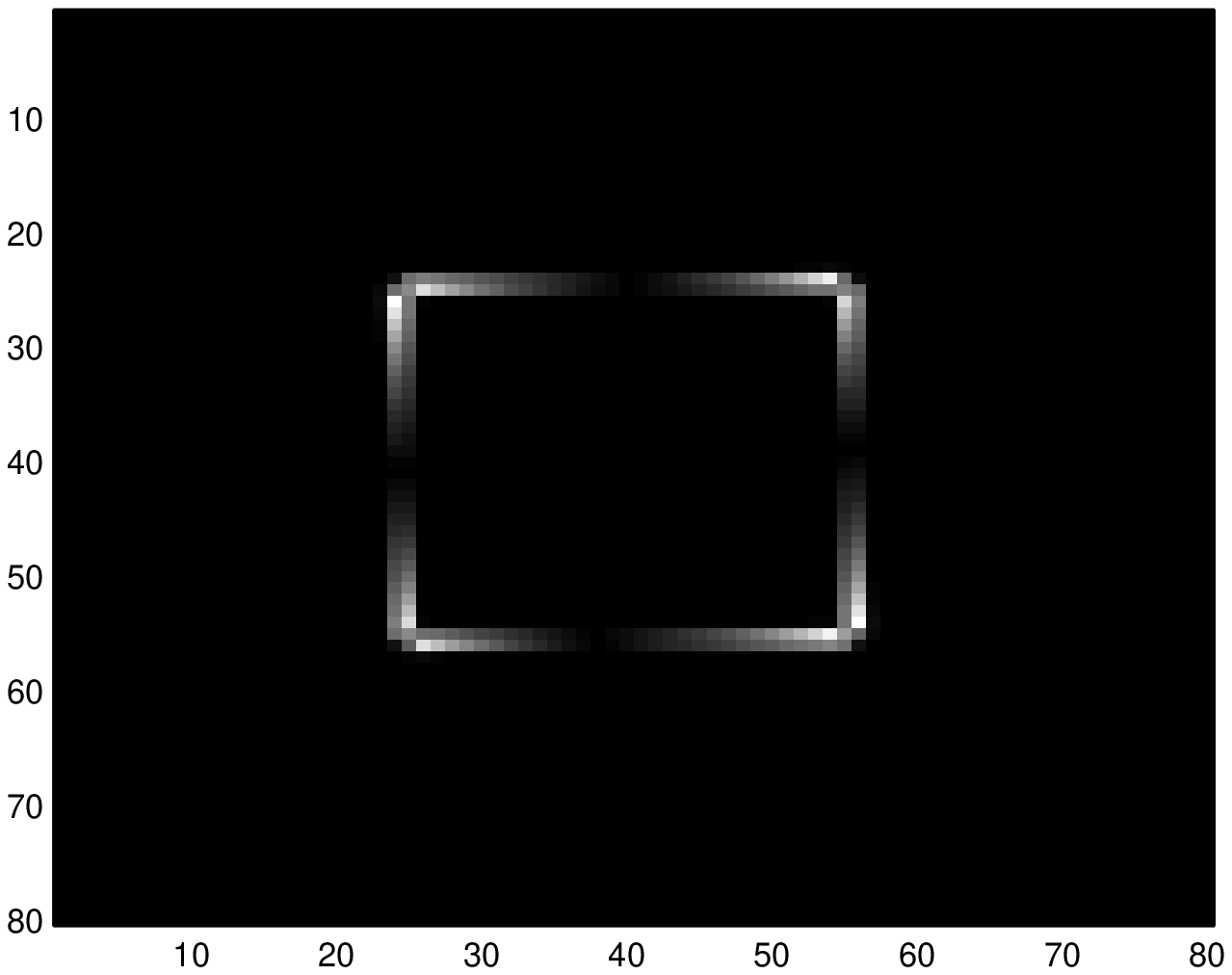}}
  }
\end{center}
\caption{With linear transformations, the dimensionality of the higher-order representation 
matches the dimensionality of the transformation.
A dilation (e) and rotation (d) is applied to the fixed binary images (b) and
(f), respectively. The registration results (c) and (g) subtracted from the
fixed images are shown in the difference pictures (d) and (h). The
registration is performed with a single first-order momenta in the center of the
pictures, and the number of parameters for the registration thus matches the
dimensionality of the linear representations. The slight differences between results and fixed images are caused by the
first-order approximation in \eqref{eq:tildeU1}. Increasing the kernel size,
adding more control points, or using second order momenta would imply less
difference.
}
\label{fig:simple-images}
\end{figure}

\subsection{Articulated Motion}
The articulated motion of the finger\footnote{X-ray frames from
\url{http://www.archive.org/details/X-raystudiesofthejointmovements-wellcome}} in Figure~\ref{fig:finger} (a) and (b)
can be described by three locally linear transformations. With higher-order
momenta, we can place deformation atoms at the center of the bones in the moving
and fixed images, and use the point positions together with the direction of the
bones to drive a registration. This natural and low dimensional representation
allows a fairly good match of the images resembling the use of
the Polyaffine affine framework for
articulated registration \cite{seiler_geometry-aware_2011}. A similar
registration using zeroth order momenta would
need two deformation atoms per bone and lacking a natural way to place such
atoms, the positions would need to be optimized. With higher-order momenta, the
deformation atoms can be placed in a natural and consistent way, and the total
number of free parameters is lower than a zeroth order representation using two
atoms per bone.
\begin{figure}[t]
\begin{center}
  \parbox{0.99\columnwidth}{
  \subfigure[Moving]{\includegraphics[width=0.32\columnwidth,trim=40 0 40 0,clip=true]{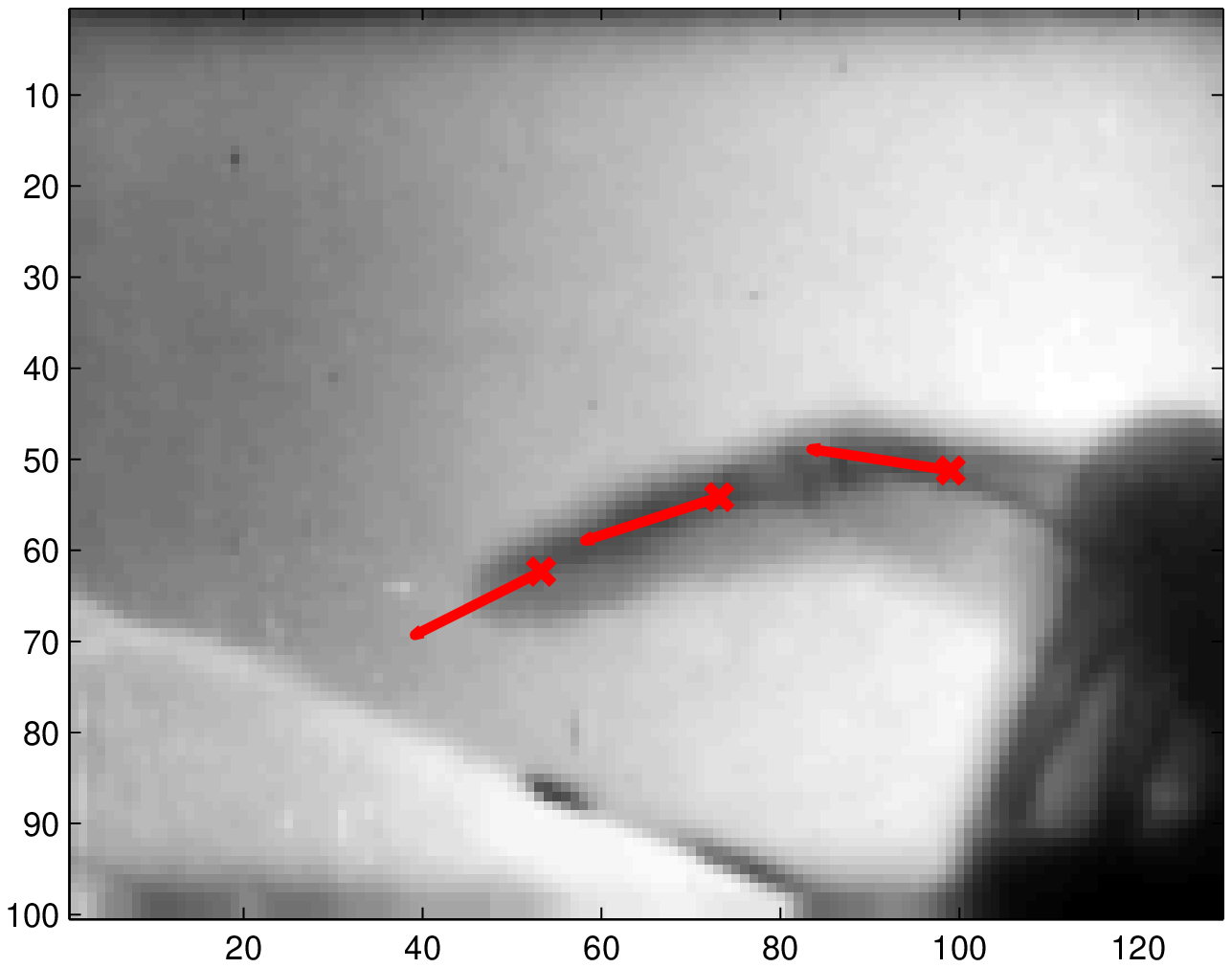}}
  \subfigure[Fixed]{\includegraphics[width=0.32\columnwidth,trim=40 0 40 0,clip=true]{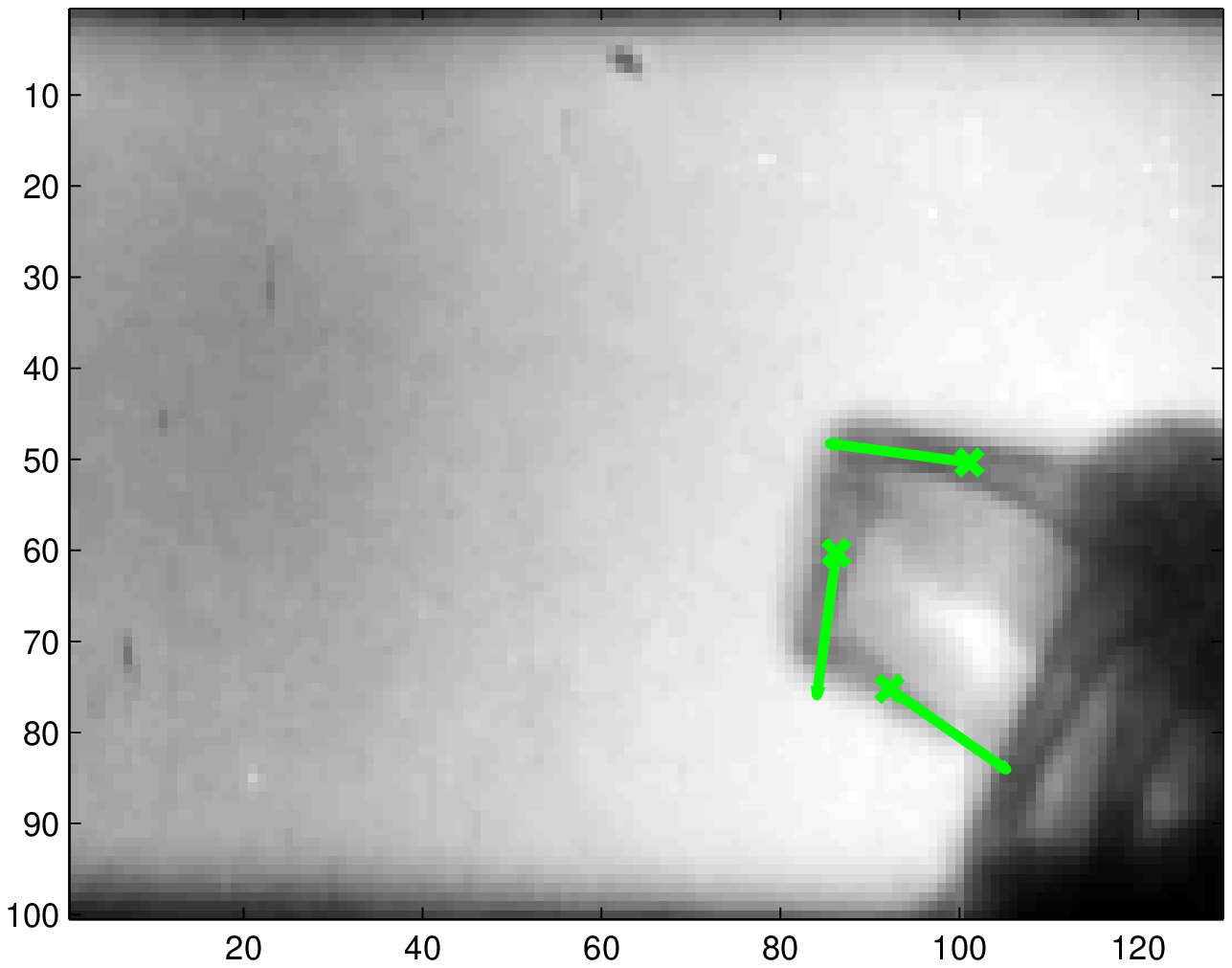}}
  \subfigure[Result, first-order]{\includegraphics[width=0.32\columnwidth,trim=40 0 40 0,clip=true]{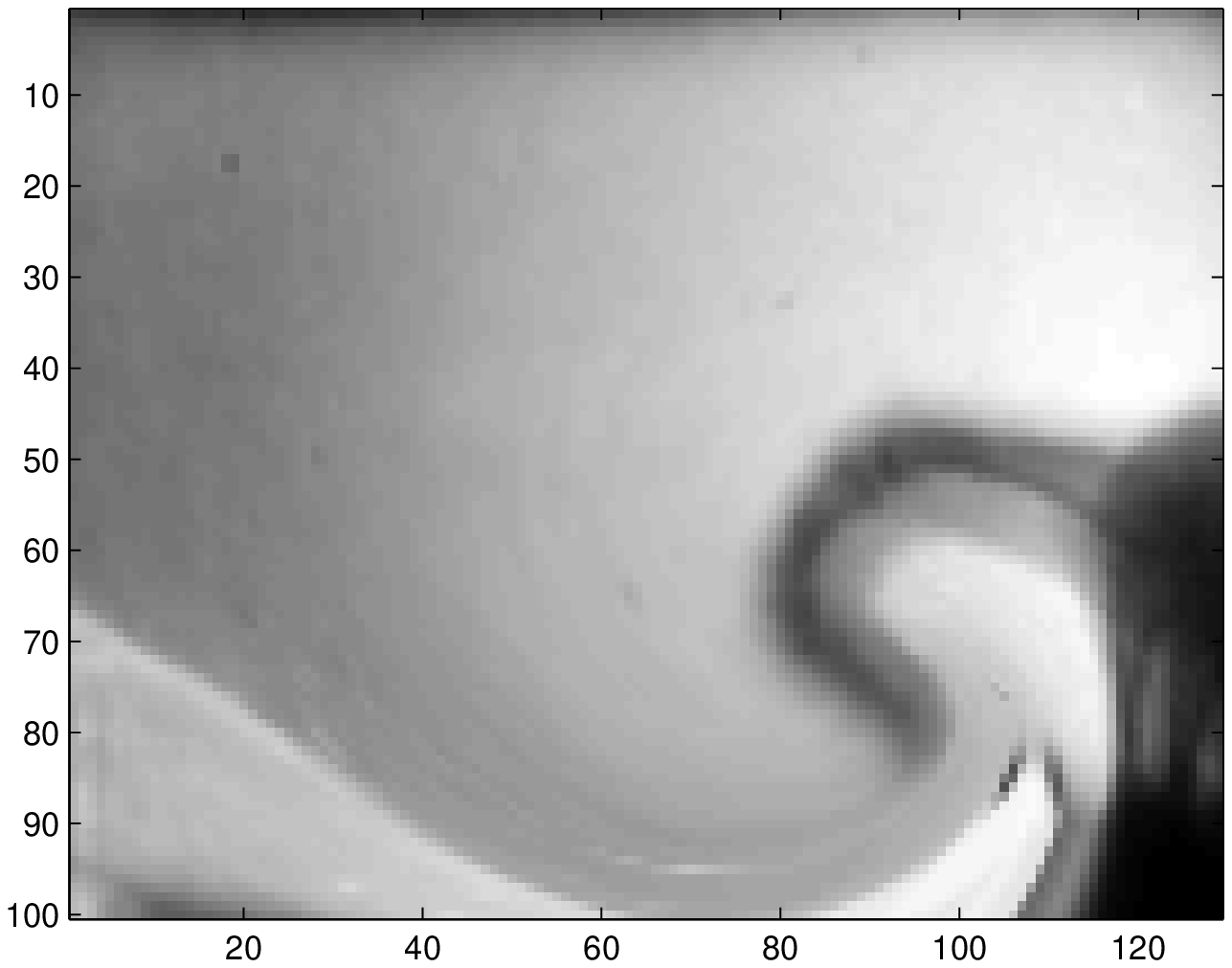}}
  % \subfigure[Result, zeroth order]{\includegraphics[width=0.24\columnwidth,trim=0 0 0 0,clip=true]{figures/finger-zero}}
  }
\end{center}
\caption{Registering articulated movement using directional information of the
bones: the landmarks and bone orientations (red points and arrows) in the moving
image (a) are matched against the
landmarks and bone orientations (green points and arrows) in the fixed image
(b). The result using first-order momenta (c) can be obtained with a low number of deformation atoms that
can be consistently placed at the center of the bones. A
corresponding zeroth order representation would use a higher number of atoms with
a corresponding increase in the number of parameters.
}
\label{fig:finger}
\end{figure}

\subsection{Registering Atrophy}
Atrophy occurs in 
the human brain among patients suffering from Alzheimer's disease, 
and the progressing atrophy can be detected by 
the expansion of the ventricles \cite{jack_medial_1997,fox_presymptomatic_1996}. 
Since first-order momenta offer compact description of expansion, this makes a 
parametrization of the registration based on higher-order momenta 
suited for describing the expansion of the ventricles, and, in addition, the
deformation represented by the momenta will be easily interpretable.
In this experiment, we therefore suggest a registration method that using few degrees of freedom
describes the expansion of the ventricles, and does so in a way that can be
interpreted when doing further analysis of e.g. the volume change.

We use the publicly available Oasis dataset\footnote{
\url{http://www.oasis-brains.org}
} \cite{marcus_open_2010}, and, in order to illustrate the use of higher-order
momenta, we select a small number of patients from which
two baseline scans are acquired at the same day together with a later
follow up scan. The patients are in various stages of dementia, and, for each
patient, we rigidly register the two baseline and one follow up scan 
\cite{darkner_generalized_2011}.

The expanding ventricles can be registered by placing deformation atoms 
in the center 
of the ventricles of the fixed image as shown in Figure~\ref{fig:atrophy1}. 
For each patient, we manually place five deformation atoms in the ventricle area
of the first baseline 3D volume. It is important to note
that though we localize the description of the deformation to the deformation
atoms, the atoms control the deformation field throughout the ventricle area.
Based on the size of the ventricles, we use 3D Gaussian kernels with a scale of 15
voxels, and we let the 
regularization weight in \eqref{eq:func-lddmm} be $\lambda=16$. The effect of
these choices is discussed below.
Each deformation atom consists of a zeroth and first-order momenta.
We use $L^1$ similarity to drive the registration
\cite{darkner_generalized_2011}\footnote{See also
  \url{http://image.diku.dk/darkner/LOI}.}
and, for each patient, we perform two registrations: we register the two baseline
scans acquired at the same day, and we register one baseline scan against the follow up scan.
Thus, the baseline-baseline registration should indicate no ventricle expansion,
and we expect the baseline-follow up registration to indicate ventricle expansion.
Figure~\ref{fig:atrophy1} shows for one patient the placement of the control
points in the baseline image, the follow up image, the $\log$-Jacobian
determinant in the ventricle area of the generated deformation, and the initial vector 
field driving the registration. 

The use of first-order momenta allows us to interpret the result of the 
registrations and to relate the results to possible expansion of
the ventricles. The volume change is indicated by the Jacobian
determinant of the generated deformation at the deformation atoms as well as by
the divergence of the first-order momenta. The latter is available directly from the registration parameters. We plot in 
Figure~\ref{fig:atrophy2} the logarithm of the
Jacobian determinant and the divergence
for both the same day baseline-baseline registrations and for the baseline-follow up registrations.
Patient $1-4$ are classified as demented, patient $5$ and $6$ as non-demented, and
all patient have constant clinical dementia rating through the experiment. The
time-span between baseline and follow up scan is 1.5-2 years with the exception
of 3 years for patient four. 
As expected, the $\log$-Jacobian is close to zero for the same day
baseline-baseline scans but it increases with the baseline-follow up
registrations of the demented patients. In addition,
the correlation between the $\log$-Jacobian and the divergence shows how the
indicated volume change is related directly to the registration parameters;
the parameters of the deformation atoms can in this way be directly interpreted
as encoding the amount of atrophy.
\begin{figure}[t]
  \parbox{0.99\columnwidth}{
\begin{center}
  \subfigure[The average $\log$-Jacobian of the final deformation at the 5 deformation atoms for the baseline-baseline and baseline-follow up registrations]{\includegraphics[width=0.46\columnwidth,trim=0 0 0 0,clip=true]{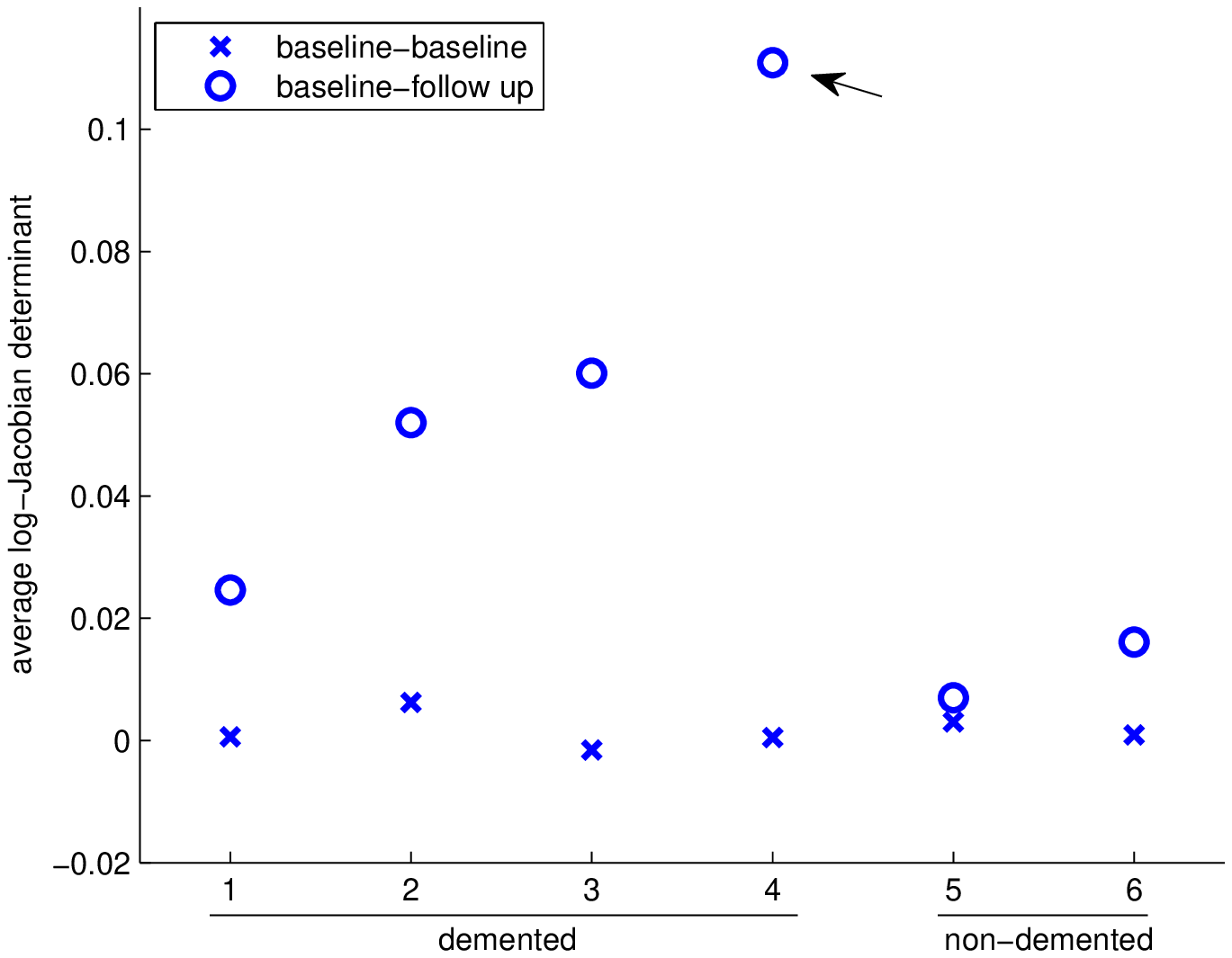}}
  \hspace{1em}
  \subfigure[The average divergence at the deformation atoms for the baseline-baseline and baseline-follow up registrations]{\includegraphics[width=0.46\columnwidth,trim=0 0 0 0,clip=true]{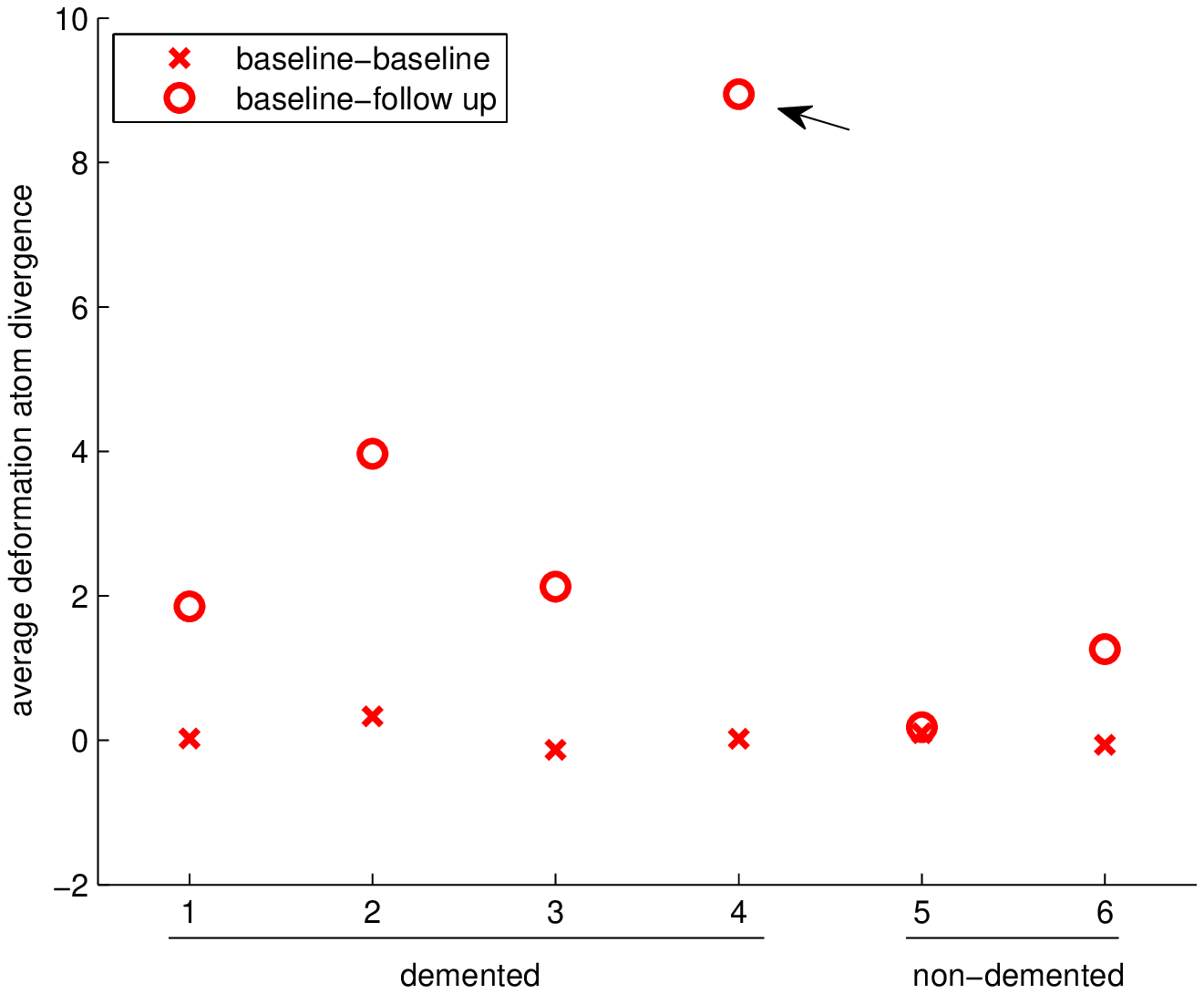}}
  %\subfigure[Difference after registration, ventricle area of patient 2.]{\includegraphics[width=0.47\columnwidth,trim=0 0 0 0,clip=true]{figures/diff-baseline-followup}}
  %\hspace{1em}
  %\subfigure[$\log$-Jacobian in ventricle area of patient 2.]{\includegraphics[width=0.47\columnwidth,trim=0 0 0 0,clip=true]{figures/detjac}} 
\end{center}
  }
\caption{Indicated volume change: (a) The average $\log$-Jacobian determinant of
the generated deformation at 
the 5 deformation atoms for six patients (1-4 demented, 5-6 non-demented); 
(b) divergence of the 5 higher-order momenta representing the deformation. 
The divergence can be extracted
directly from the parameters of the first-order momenta, and the correlation
between the $\log$-Jacobian and the divergence as seen by the similarity between
(a) and (b) therefore shows the interpretability of the deformation atoms.
The time-span between baseline and follow up scans are 1.5-2 years with the exception
of 3 years for patient four (arrows). 
}
\label{fig:atrophy2}
\end{figure}

We chose two important parameters above: the kernel scale and the
regularization term. The choice of one scale for all patients works well if the 
ventricles to be registered are of approximately the same size at the baseline
scans. If the ventricles vary in size, the scale can be chosen individually for
each patient. Alternatively, a multi-scale approach could do this automatically
which suggests combining the method with e.g.
the kernel bundle framework \cite{sommer_multi-scale_2011}.
Depending on the image forces, the regularization term in \eqref{eq:func-lddmm}
will affect the amount of expansion captured in the registration.
Because of the low number of control points, we can in practice set the
contribution of the regularization term to zero without experiencing
non-diffeomorphic results. It will be interesting in the future to estimate the
actual volume expansion directly using the parameters of the deformation atoms with this less biased model.

\section{Conclusion and Outlook}
\label{sec:concl}
We have introduced higher-order momenta in the LDDMM registration framework. The
momenta allow \emph{compact} representation of locally affine transformations by
increasing the \emph{capacity} of the deformation description. Coupled with
similarity measures incorporating first-order information, the higher-order
momenta improve the range of deformations reached by sparsely discretized LDDMM
methods, and they allow direct capture of first-order information such as
expansion and contraction. In addition, the constitute deformation atoms for which
the generated deformation is directly interpretable.

We have shown how the partial derivative reproducing property implies singular momentum for the
higher-order momenta, and we used this to derive the EPDiff evolution equations. By
computing the forward and backward variational equations, we are able to transport gradient
information and derive a matching algorithm. We provide examples showing typical
deformation coded by first-order momenta and how images can be registered using
a very few parameters, and we
have applied the method to register human brains with progressing atrophy.

The experiments included here show only a first step in the application
of higher-order momenta: the representation may be applied to register entire images;
merging the method with multi-scale approaches will increase the description capacity 
and may lead to further reduction in the dimensionality
of the representation. Combined with efficient implementations, higher-order
momenta promise to provide a step forward in compact deformation description for
image registration.

\bibliographystyle{siam}
\bibliography{bibliography.bib}

\end{document}